\documentclass[11pt]{article}

\usepackage[T1]{fontenc}
\usepackage[full]{textcomp}
\usepackage{newtxtext}
\usepackage{cabin} 
\usepackage[varqu,varl]{inconsolata} 
\usepackage[english]{babel}
\usepackage{amsmath}
\usepackage{amssymb}

\usepackage{scalefnt}
\usepackage{letltxmacro}
\LetLtxMacro{\oldtextsc}{\textsc}
\renewcommand{\textsc}[1]{\oldtextsc{\scalefont{1.1}#1}}

\usepackage{amsfonts}       
\usepackage[nice]{nicefrac}

\usepackage[usenames,dvipsnames]{xcolor}
\definecolor{shadecolor}{gray}{0.9}

\usepackage{afterpage}
\usepackage{framed}

\usepackage{lineno}

\usepackage{ragged2e}

\newcounter{parcount}

\usepackage{graphicx}
\usepackage[labelfont=bf]{caption}
\usepackage[format=hang]{subcaption}

\usepackage{booktabs}

\usepackage[algoruled]{algorithm2e}
\usepackage{listings}
\usepackage{fancyvrb}
\fvset{fontsize=\normalsize}

\usepackage{natbib}

\usepackage[colorlinks,linktoc=all]{hyperref}
\usepackage[all]{hypcap}
\hypersetup{citecolor=MidnightBlue}
\hypersetup{linkcolor=MidnightBlue}
\hypersetup{urlcolor=MidnightBlue}

\usepackage[capitalize,noabbrev,nameinlink]{cleveref}

\usepackage[acronym,nowarn]{glossaries}

\lstdefinestyle{mystyle}{
    commentstyle=\color{OliveGreen},
    keywordstyle=\color{BurntOrange},
    numberstyle=\tiny\color{black!60},
    stringstyle=\color{MidnightBlue},
    basicstyle=\ttfamily,
    breakatwhitespace=false,
    breaklines=true,
    captionpos=b,
    keepspaces=true,
    numbers=left,
    numbersep=5pt,
    showspaces=false,
    showstringspaces=false,
    showtabs=false,
    tabsize=2
}
\newcommand{\E}{\mathbb{E}}

\newcommand{\R}{\mathbb{R}}

\newcommand{\N}{\mathcal{N}}

 \newacronym{KL}{kl}{Kullback-Leibler}
\newacronym{ELBO}{elbo}{\emph{evidence lower bound}}
\newacronym{POPELBO}{pop-elbo}{\emph{population evidence lower bound}}

\newacronym{SVI}{svi}{stochastic variational inference}
\newacronym{BUMPVI}{bump-vi}{bumping variational inference}

\newacronym{GMM}{gmm}{Gaussian mixture model}
\newacronym{LDA}{lda}{latent Dirichlet allocation}

\newacronym{SUTVA}{sutva}{stable unit treatment value assumption}
 \usepackage[makeroom]{cancel}

\usepackage{xspace}

\usepackage{tikz}
\usetikzlibrary{fit,positioning,calc}
\definecolor{light}{RGB}{220, 188, 188}
\definecolor{mid}{RGB}{185, 124, 124}
\definecolor{dark}{RGB}{143, 39, 39}
\definecolor{highlight}{RGB}{0, 255, 0}
\definecolor{gray10}{gray}{0.1}
\definecolor{gray20}{gray}{0.2}
\definecolor{gray30}{gray}{0.3}
\definecolor{gray40}{gray}{0.4}
\definecolor{gray60}{gray}{0.6}
\definecolor{gray70}{gray}{0.7}
\definecolor{gray80}{gray}{0.8}
\definecolor{gray90}{gray}{0.9}
\definecolor{gray95}{gray}{0.95}
\definecolor{comment}{gray}{0.50}

\creflabelformat{equation}{#1#2#3}

\makeatletter
\renewcommand{\@makefnmark}{\hbox{\textsuperscript{\tiny{\@thefnmark}}}}
\makeatother

\usepackage[parfill]{parskip}
\usepackage[top=1in, bottom=1in, left=1.3in, right=1.3in]{geometry}
\usepackage{setspace}
\allowdisplaybreaks

\usepackage{makecell}
\usepackage{amsthm}
\newtheorem{theorem}{Theorem}
\usepackage{newtxmath}

\title{\textbf{Estimating Wage Disparities Using Foundation Models}}

\author{\begin{tabular}{ccc}\textbf{Keyon Vafa}  & \textbf{Susan Athey} & \textbf{David M. Blei} \\Harvard University & Stanford University & Columbia University\\
\end{tabular} \date{}}

\begin{document}
\setlength{\tabcolsep}{14pt}
\maketitle
\setlength{\tabcolsep}{6pt}

\begin{abstract}
\noindent
The rise of foundation models marks a paradigm shift in machine learning: instead of training specialized models from scratch, foundation models are first trained on massive datasets before being adapted or fine-tuned to make predictions on smaller datasets. 
Initially developed for text, foundation models have also excelled at making predictions about social science data. 
However, while important social science problems use prediction as an intermediate step, they ultimately require different criteria for success, such as when estimating causal effects or decomposing group differences into explained and unexplained components. 
In this paper, we develop methods for fine-tuning foundation models to perform these estimation problems. 
We first characterize an omitted variable bias that can arise when a foundation model is only fine-tuned to maximize predictive accuracy, a common approach in machine learning. We then provide a novel set of conditions for fine-tuning under which estimates derived from a foundation model are $\sqrt{n}$-consistent. Based on this theory, we develop new fine-tuning algorithms that empirically mitigate this omitted variable bias. 
To demonstrate our ideas, we study gender wage decomposition. This is a statistical estimation problem from econometrics where the goal is to decompose the gender wage gap into components that can and cannot be explained by career histories of workers.
Classical methods for decomposing the wage gap employ simple predictive models of wages which condition on coarse summaries of career history that may omit factors that are important for explaining the gap. 
Instead, we use a custom-built foundation model to decompose the gender wage gap, which captures a richer representation of career history than simple models.
Using data from the Panel Study of Income Dynamics, we find that career history explains more of the gender wage gap than standard econometric models can measure, and we identify elements of career history that are omitted by standard models but are important for explaining the wage gap.

 \end{abstract}

Foundation models have revolutionized the machine learning approach to prediction \citep{devlin2018bert,radford2019language,bommasani2021opportunities}. In contrast to traditional predictive models, which are trained to make predictions on specific, individual tasks, foundation models are typically trained in two steps: they first are trained on massive, passively-collected datasets and then are adapted to specific tasks. The success of these models stems from their ability to transfer information learned during the initial training period to new prediction problems through approaches like supervised fine-tuning --- adjusting a model's parameters to minimize prediction error on labeled examples from a target task \citep{devlin2018bert}. For example, large language models \citep{devlin2018bert,radford2019language} are foundation models that were originally trained to predict the next word of Internet articles, but can be fine-tuned to make other predictions involving text, like the next word of a conversation or the sentiment of a movie review.

While foundation models have been successful at making predictions about social science data \citep{vafa2023career,savcisens2024using}, many core problems in social science require more than just accurate predictions. 
For example, social scientists often aim to estimate causal effects under the assumption of unconfoundedness \citep{imbens2004nonparametric} or decompose observed differences between groups into explained and unexplained components based on observable factors \citep{kitagawa1955components,oaxaca1973male,blinder_wage_1973} --- 
isomorphic problems that use prediction as an intermediate step but ultimately require different criteria for success. While fine-tuning foundation models may be useful for these analyses, optimizing for predictive accuracy alone does not guarantee valid decompositions or causal estimates. 

In this paper, we develop methods for adapting foundation models to perform decomposition and causal effect estimation by modifying how they are fine-tuned. Rather than fine-tuning foundation models to minimize prediction error, we develop objectives specifically designed for these estimation problems. Our first contribution is characterizing a statistical bias that arises when a foundation model discards information that may not be important for prediction but is relevant for the estimation problem. We then provide a novel set of conditions for fine-tuning under which estimates derived from a foundation model are not only unbiased but also consistent at a fast asymptotic rate. These conditions motivate new \textit{debiased fine-tuning} methods. Our key insight is that fine-tuning foundation models for these applications requires addressing an omitted variable bias that standard supervised fine-tuning does not address.

To demonstrate these ideas, we focus on an application that addresses a classic decomposition problem from labor economics: estimating the difference between how individuals with the same labor market experience get paid when they belong to different demographic groups (see \citet{blau_gender_2017} and \citet{altonji_chapter_1999} for reviews). Accurately estimating this \textit{unexplained wage gap} is important to help guide policy for reducing disparities. But the unexplained gap is challenging to estimate with traditional econometric models. It involves predicting an individual's wage from their labor market history, a high-dimensional and complicated variable. Our paper demonstrates that foundation models of labor market history can improve the predictions that underlie wage gap estimates. 

We use CAREER, a foundation model of labor market history \citep{vafa2023career}, to estimate unexplained wage gaps. CAREER is initially fit to a massive resume dataset to predict the next job an individual will have, rather than their wage. Naively, we can fine-tune CAREER to make accurate predictions of wage on the datasets used for wage gap estimation. However, using this approach to estimate the unexplained wage gap can amplify a classical problem: omitted variable bias. 
Instead, we develop new \textit{debiased fine-tuning} methods to fine-tune foundation models so they can properly estimate unexplained wage gaps. 
The key is to fine-tune foundation models not to minimize predictive error but rather to reduce omitted variable bias. 
In semi-synthetic experiments, we show that debiased fine-tuning methods form better estimates of the unexplained wage gap than the standard fine-tuning approach.

We use our methods to estimate the explained gender wage gap on survey data from the Panel Study of Income Dynamics (PSID) \citep{psid}. We first demonstrate that foundation models form accurate predictions of wage and gender; they outperform standard econometric models for predicting wage by 10-15\%. We then use debiased fine-tuning methods to estimate the gender wage gap. We find that history consistently explains more of the gap than the variables typically included in standard econometric models. We conclude by studying which aspects of work history, captured by foundation models but omitted from prior approaches, are important for explaining the wage gap. 

While this paper studies unexplained wage gaps in detail, the results and methods we develop are applicable to a broader set of problems, such as causal estimation. In particular, as observed by \citet{fortin2011decomposition} and others in the literature, the problem of estimating a decomposition of a wage gap into explained and unexplained components is isomorphic to the problem of estimating the average effect of a treatment under the assumption of unconfoundedness. Although the interpretation of the estimate is distinct for decompositions, 
the statistical theory that applies to estimation is the same (see \citet{imbens2015causal} for a review).  Thus, our results also provide new theory and methods for the problem of incorporating foundation models into the estimation of treatment effects.

Relative to both the causal inference and decomposition literatures, our theory is adapted to a scenario where a foundation model may bring in information from a distinct, larger dataset, and where we fine-tune the model to avoid omitted variable bias.  If we solve the latter problem well enough, then the traditional semi-parametric theory (e.g. \citet{chernozhukov2018double}) can be applied as if the representations of high-dimensional covariates derived from the fine-tuned foundation model are sufficient statistics for the full high-dimensional covariate vector.  The methods we introduce thus provide a widely applicable new framework for leveraging the capabilities of foundation models while mitigating biases due to omitted variables that they may introduce.

\section{Explaining Wage Gaps with Foundation Models}
\label{sec:framework}
The unexplained wage gap is the wage gap between two groups of individuals with the same observed characteristics. We estimate an unexplained wage gap that arises when individuals in different groups have the same labor market history. 

Consider the gender wage gap. 
In the United States, females earn roughly 80\% the male hourly wage \citep{blau_gender_2017}. Motivated by the fact that the male and female labor forces differ in observable ways, a large literature seeks to explain this wage gap through differences in these observable factors \citep{kitagawa1955components,blinder1973wage,oaxaca1973male,blau_gender_2017}. 
One of the most important factors for explaining the gender wage gap is differences in the number of years that males and females have spent in the labor force \citep{blau_gender_2017}. 
Understanding the difference in wages between males and females with the \textit{same career histories} can help guide policy: if the unexplained gap is large, attempts to close the gap may involve interventions to address problems in bargaining or fairness in wage setting. On the other hand, if the gap can be accounted for by gender differences in career histories, these interventions might target career pathways. 

More generally, consider $N$ individuals, indexed by $i = 1, \dots, N$. Each individual belongs to a binary group $A_i \in \{0, 1\}$ (e.g. $A_i=0$ denotes males and $A_i=1$ denotes females). Each individual also has a career history $X_i$, a sequence of $T$ discrete occupations and years, $X_i=((J_{i1}, D_{i1}), \dots, (J_{iT},D_{iT})) \in \mathcal X$, and where each occupation label $J_{it} \in \{1, \dots, N_J\}$ encodes the occupation an individual worked in during year $D_t$ (or their labor status if they're not working, e.g. ``unemployed'' or ``student''). Finally, denote an individual's log-wage by $Y_i \in \R$. Each individual is sampled i.i.d. from a joint distribution $P(X, A, Y)$. Define the conditional expectation function, $\mu_a(x) = \E_P[Y|A=a, X=x]$, and the propensity function, $e(x) = P(A=1|X=x)$.

The raw \textbf{wage gap} is
\begin{equation*}
\text{WG} = \E_{p(x|a=1)}[\mu_1(X)] - \E_{p(x|a=0)}[\mu_0(X)].
\end{equation*}
This is the difference in the average wage between the two groups. 
Our goal is to estimate the wage gap that is \textbf{unexplained} by history: 
\begin{equation*}
\label{eqn:unexplained_gap}
\text{UWG} = \E_{p(x|a=1)}[\mu_1(X) - \mu_0(X)].
\end{equation*}
This is the average difference in the expected wage between individuals in the two groups who have the same career histories. 
The unexplained and raw wage gaps are linked by a classic decomposition \citep{kitagawa1955components,blinder1973wage,oaxaca1973male}:
\begin{equation}
\begin{aligned}
\text{WG} &= \underbrace{\E_{p(x|a=1)}\left[\mu_1(X) - \mu_0(X)\right]}_{\text{unexplained wage gap}} + \underbrace{\E_{p(x|a=1)}\left[\mu_0(X)\right] - \E_{p(x|a=0)}\left[\mu_0(X)\right]}_{\text{explained wage gap}}
\end{aligned}
\label{eqn:blinder_oaxaca_decomposition}
\end{equation}
The unexplained wage gap is the portion of the raw wage gap that cannot be attributed to gender differences in career histories. 

The explained wage gap could be conditioned on factors in addition to labor market history, e.g. an individual's educational background. For simplicity our notation only includes history, but we incorporate additional observed characteristics of individuals in our empirical analyses. We also note that the unexplained wage gap is only non-parametrically identifiable under an overlap condition (see, e.g. \citet{imbens2004nonparametric}): $P(A=1 | X) < 1$. We assume overlap throughout this paper and we limit the sample of workers we analyze empirically to those with histories where the condition is satisfied.

\subsection*{Foundation models can improve predictions}
A common approach for estimating the unexplained wage gap involves constructing an estimator  $\hat \mu_g(x)$ for $\mu_g(x)$, which in turn can be used to form an estimate of the unexplained wage gap:
\begin{equation}
\label{eqn:outcome_only_model}
\textstyle \frac{1}{N_1} \sum_i A_i * (\hat\mu_1(X_i) - \hat\mu_0(X_i)),
\end{equation}
where $N_1 = \sum_i A_i$. 

However, estimating the relationship between history and wage is challenging with realistic data sizes because career histories are high-dimensional: the number of possible career histories grows exponentially in the number of years someone has worked. This challenge is compounded by the fact that unexplained wage gaps are commonly estimated using small survey-based datasets, particularly in the U.S. where administrative data about worker histories are not generally available to researchers. For this reason, traditional econometric approaches have used a small number of hand-constructed summaries of labor market experience, 
which keeps the number of covariates in the predictive model small relative to the data set size. But these summary statistics do not capture the full complexity of labor market history, and in particular they may omit factors of history that are important for explaining the wage gap. 
~\looseness=-1

For example, a large body of literature has focused on decomposing the gender wage gap in the United States by applying \Cref{eqn:blinder_oaxaca_decomposition} to small survey datasets (see \citet{blau_gender_2017} and \citet{altonji_chapter_1999}  for surveys). 
Rather than including an individual's career history, most analyses include summary statistics about an individual's career history, such as the years of experience or tenure in the current job \citep{mincer1974family,manning2008gender,blau2013feasibility,blau_gender_2017}. Even though many occupational taxonomies contain hundreds of fine-grained categories, it is most common to include coarse-grained occupational categories containing 20-30 categories \citep{blau1999analyzing,blau_gender_2017,boheim2021decomposition,hegewisch2014occupational}. 
Because these models rely on a relatively small number of covariates, $\hat \mu_0$ and $\hat \mu_1$ are typically constructed using relatively simple models models, such as linear regressions \citep{oaxaca1973male,blinder_wage_1973,blau1999analyzing,blau_gender_2017} or LASSO models \citep{boheim2021decomposition,bonaccolto2022gender}. 

However, these incomplete measures of experience discard factors that help explain the wage gap. 
For example, \citet{regan2009work} and \citet{blau2013feasibility} find that potential experience (an inexact measure of experience that does not measure workforce interruptions) explains less of the wage gap than years of actual work experience. 
Moreover, \citet{light1995early} estimates a wage model with detailed measures of year-by-year experience, finding that the timing of work experience explains a substantial portion of the wage gap. 
While incorporating full histories into gender wage gap analyses could ensure these factors aren't discarded, the predictive models used for gender wage gap analyses are too simple to include them.

Foundation models \citep{bommasani2021opportunities} offer an alternative approach. Foundation models are machine learning models that \textit{learn} low-dimensional representations of high-dimensional variables from data. 
These representations are initially learned on massive, passively-collected data after which they can be adapted on specific datasets of interest. 
For example, in natural language processing, foundation models that are trained to predict words using terabytes of Internet text can be adapted to generate responses to human questions \citep{achiam2023gpt,touvron2023llama}. While initially developed for text, foundation models have successfully addressed seemingly intractable prediction problems in domains such as computer vision \citep{dosovitskiy2020image}, music \citep{huang2018music}, and protein generation \citep{madani2023large} 

A foundation model of labor market history can help estimate the unexplained wage gap by providing a low-dimensional representation of history that is predictive of wage. Because representations are learned from data, these estimates are not limited to the features a researcher knows to include.

Formally define a representation to be a function $\lambda(X): \mathcal{X} \to \R^D$. Given a representation $\lambda$, the wage gap unexplained by the representation of history is
\begin{equation*}
\text{UWG}(\lambda) = \E_{p(x|a=1)}[\mu_{1}(\lambda(X)) - \mu_{0}(\lambda(X))],
\end{equation*}
where $\mu_{a}(\lambda(x))=\E[Y|A=a,\lambda(X)=\lambda(x)]$ is the expected wage as a function of the representation $\lambda(x)$. 

For the rest of the paper, we consider estimating the unexplained wage gap using CAREER, a foundation model of labor market history \citep{vafa2023career}. CAREER is trained to learn representations that can predict the next occupation a worker will have from a dataset of 24 million resumes posted online. When these representations are adapted to small survey datasets, CAREER makes more accurate predictions of an individual's next occupation than standard econometric approaches. We consider using these representations to predict an individual's wage. 
~\looseness=-1

\subsection*{Foundation models can introduce omitted variable bias} Foundation models are effective because they compress high-dimensional information into low-dimensional representations. However, we demonstrate that replacing an individual's history with a representation can introduce an omitted variable bias \citep{chernozhukov2022long}.

We say the wage gap unexplained by a representation $\lambda$ is biased if it differs from the wage gap unexplained by the full history. Define this bias as $\text{OVB}(\lambda) = \text{UWG}(\lambda) - \text{UWG}$. It has a closed-form expression: 
\begin{equation}
\label{eqn:ovb}
\text{OVB}(\lambda) = \E_{p(x,a)}\big[(\mu_{A}(\lambda(X))-\mu_A(X))) 
                      *(\alpha_{A}(\lambda(X)) - \alpha_A(X))\big] 
\end{equation}
where
\begin{align*}
\textstyle \alpha_a(x) &= \textstyle -\frac{1-a}{P(A=1)}\left(\frac{e(x)}{1-e(x)}\right),\\ 
\alpha_{a}(\lambda(x)) &= \textstyle -\frac{1-a}{P(A=1)}\left(\frac{e(\lambda(x))}{1-e(\lambda(x))}\right),
\end{align*}
for representation-based propensity function $e(\lambda(x))=P(A=1|\lambda(X)=\lambda(x))$. 
\Cref{app:ovb_derivation} contains a detailed derivation. \Cref{eqn:ovb} can also be seen as a special case of the general omitted-variable-bias formula in \citet{chernozhukov2022long}, where $\alpha(\lambda(X))$ and $\alpha(X)$ correspond to the short and long Riesz representers, respectively.

\Cref{eqn:ovb} provides intuition for how a representation can induce bias. The omitted variable bias is a covariance of two differences: the first term is the difference in expected wage as a function of history and the representation of history, while the second term is the difference in the group propensity odds ratio as a function of history and the representation of history. A low-dimensional representation by definition discards information; for there to be no omitted variable bias, the discarded information that's related to wage should be unrelated to group propensity, and vice-versa. \Cref{eqn:ovb} is closely related to econometric results about the extent of omitted variable bias in semiparametric models (see e.g. \citet{chernozhukov2022long}); while this literature focuses on whether individual variables are included or not in models, we focus on representations, which can still omit variables despite being functions of all variables. Focusing on representations, \citet{veitch2020adapting} provide a sufficient condition under which \Cref{eqn:ovb} is 0, but do not characterize the exact level of bias; meanwhile, \Cref{eqn:ovb} exactly characterizes the extent of omitted variable bias.

\subsection*{Debiasing foundation models}
Although CAREER is a foundation model that is trained to learn representations from data, it is not trained to minimize omitted variable bias (\Cref{eqn:ovb}). One reason is that its representations are trained to optimize a single objective that doesn't naturally appear in \Cref{eqn:ovb}: the predictability of an individual's next job. Moreover, the representations are trained on a different population of individuals than those for whom we'd like to estimate the unexplained wage gap. 

Even biased, a foundation model can still be useful for estimating the unexplained wage gap because it can be fine-tuned. Empirically, when a foundation model's representations are adjusted to optimize a related but distinct objective from the one they were initially trained to optimize, they often outperform models trained on only the new objective \citep{devlin2018bert,lewis2019bart}. We do not have to learn unbiased representations of career history from scratch; we can adjust the representations of a pretrained foundation model to debias it. 

The standard approach for modifying foundation models is \textbf{supervised fine-tuning} \citep{devlin2018bert}. In our setting, supervised fine-tuning would entail modifying a foundation model's representation $\lambda$ to be predictive of wage on the survey data used for wage gap estimation. But while a foundation model would likely form better wage predictions after supervised fine-tuning, it can still be biased for estimating the unexplained wage gap; unless the foundation model recovers the exact relationship between labor market history and wage, supervised fine-tuning can still introduce arbitrarily large omitted variable bias. 

We now describe a set of conditions for fine-tuning under which an estimator of a wage gap that conditions on representations derived from a foundation model is not only unbiased and consistent but also converges at a rate proportional to $n^{-1/2}$:

\begin{theorem}
\label{eqn:theorem_ovb}
Consider a sequence of wage models $\hat \mu_{n,0}:\R^D \to \R$, propensity models $\hat e_n:\R^D \to (0, 1)$, and representations $\lambda_n: \mathcal X \to \R^D$. Denote by $\psi$ the true wage gap unexplained by history and by $\hat \psi_n$ the representation-based Augmented Inverse Probability Weighted (AIPW) estimator of the unexplained wage gap from $n$ i.i.d. samples $(X_i, A_i, Y_i) \sim P$:
\begin{equation}    
    \label{eqn:aipw}
    \hat \psi_n = \textstyle \frac{1}{\sum_i A_i} \sum_i \left(A_i - \frac{(1-A_i)\hat e_n(\lambda_n(X_i))}{1-\hat e_n(\lambda_n(X_i))}\right)(Y_i - \hat \mu_{n,0}(\lambda_n(X_i)).
\end{equation}

Assume the following:
\begin{enumerate}

\item Omitted variable bias (\Cref{eqn:ovb}) goes to 0 at a $\sqrt{n}$-rate:
\begin{equation*}
\text{OVB}(\lambda_n) = o_P(n^{-1/2}).
\end{equation*}

\item Combined $\sqrt{n}$-consistency of wage/propensity models \emph{as a function of the representation}:
\begin{align*}
&\|\hat e_n(\lambda_n(X)) - e(\lambda_n(X))\| *\|\hat \mu_{n,0}(\lambda_n(X)) - \mu_0(\lambda_n(X))\| = o_P(n^{-1/2}).
\end{align*}

\item The representations $\lambda_n$ converge to a representation $\lambda^*$ in the sense that
\begin{equation*}
    \textstyle \frac{1}{n}\sum_i (\varphi_{\lambda_n}(X_i, A_i, Y_i; \psi_{\lambda_n}) - \varphi_{\lambda^*}(X_i, A_i, Y_i; \psi_{\lambda^*})) = o_P(n^{-1/2}),
\end{equation*}
with $\text{Var}(\varphi_{\lambda^*}(X,A,Y))< \infty$, where $\varphi_{\lambda}$ is the representation-based influence function:
\begin{align*}
\textstyle \varphi_{\lambda}(X, A, Y; \psi_\lambda) &= \textstyle \frac{1}{P(A=1)}\bigg[\left(A-\frac{(1-A)e(\lambda(X))}{1-e(\lambda(X))}\right) * (Y-\mu_0)(\lambda(X)) - A \psi_\lambda\bigg]
\end{align*}
and $\psi_\lambda$ is the true gap unexplained by a representation $\lambda$
\begin{equation*}
    \psi_\lambda = \E_{P}[\mu_1(\lambda(X)) - \mu_0(\lambda(X))].
\end{equation*}

\item Additional assumptions in \Cref{app:proof_of_theorem}: cross-fitting ($\hat \mu_n$, $\hat e_n$, and $\lambda_n$ are estimated on a different sample than those used to construct $\hat \psi$); consistency of wage and propensity models as functions of the representations; strict overlap; and boundedness of wage model errors.

\end{enumerate}
Then, 
\begin{equation*}
\label{eqn:consistency_representation}
\textstyle \sqrt{n}(\hat \psi_n - \psi) \to \N\left(0, \text{Var}(\varphi_{\lambda^*}(X, A, Y; \psi))\right).
\end{equation*}

\end{theorem}

The first condition is about omitted variable bias: it requires that the omitted variable bias of the representations converges to 0 at a rate proportional to $n^{-1/2}$. 
This will be trivial for some representations: for example, $\lambda(X) = X$ has no omitted variable bias by definition. However, the second assumption imposes restrictions about modeling wage and group membership: these models must approximate the true relationship between the representation of history and these outcomes such that error goes to zero at a combined root-n rate. Note that these modeling assumptions are \textit{with respect to the representation}: the true relationships between history and the outcomes do not need to be reconstructed, only those between the representation and the outcome. Therefore, satisfying the first two assumptions involves striking a balance: representations must be detailed enough to not have omitted variable bias, but also low-dimensional enough so that outcomes can be efficiently estimated as a function of the representation. \Cref{app:proof_of_theorem} contains more details and a proof.

Although \Cref{eqn:theorem_ovb} is stated in terms of the size $n$ of the single dataset used for fine-tuning, note that much larger datasets are typically used to pretrain the foundation model. In practice, training a high-dimensional representation $\lambda(X)$ from scratch on moderate-sized survey dataset would be intractable; however, the fact that a much larger dataset contributes to the initial estimation of $\lambda(X)$ suggests that it may be feasible to adequately control omitted variable bias with a larger dimensional representation $\lambda(X)$ than would be possible without the foundation model. Indeed, under a repeated-sampling framework in which both the pretraining and fine-tuning samples are repeatedly resampled, we expect most of the sampling variation to arise from the smaller fine-tuning sample rather than from the larger pretraining sample used to train the foundation model. Consequently, our formal results condition on the pretrained representation and focus on the sampling variation arising from the fine-tuning sample.

Empirical evidence in domains such as computer vision, natural language processing, and protein structure suggests that such large-scale pretraining often yields robust and transferable representations for predictions \citep{bommasani2021opportunities}. At the same time, we do not require the foundation model to deliver a perfect or ``true'' representation; we instead require that fine-tuning from the pretrained representation allows for omitted variable to go to 0 at a $\sqrt{n}$ rate. A general-purpose foundation model may still omit some information relevant for a particular downstream task, and it is precisely the role of subsequent fine-tuning to adapt the representation to the problem at hand.

\subsection*{Relationship to causal methods}
\Cref{eqn:theorem_ovb} relates to results from the causal inference literature \citep{robins1997toward,tsiatis2006semiparametric,athey2016efficient,chernozhukov2018double,kennedy2022semiparametric}. Although the unexplained wage gap is not a causal quantity, it is mathematically identical to an average treatment effect on the treated (ATT). Specifically, Assumption 2 is similar to assumptions for $\sqrt{n}$-consistency in the doubly-robust/double machine learning literature, which often assume $o(n^{-1/2})$ combined error of outcome and propensity models \citep{chernozhukov2018double,kennedy2022semiparametric}:
\begin{equation}
    \label{eqn:double_ml_assumption}
    \|\hat e_n(X) - e(X) \| * \|\hat \mu_{n,0}(X) - \mu_0(X)\| = o_P(n^{-1/2}).
\end{equation}
This is similar to Assumption 2, with one key difference: \Cref{eqn:double_ml_assumption} requires $o(n^{-1/2})$ combined error as a function of the \textit{full history}, while Assumption 2 only requires $o(n^{-1/2})$ combined error \textit{as a function of the representation} $\lambda_n(X)$. When $\lambda_n(X)$ is lower-dimensional than the full history $X$, Assumption 2 will be more realistic than \Cref{eqn:double_ml_assumption}. In fact, when $\lambda_n(X) = X$, Assumptions 2 and 3 are trivially satisfied and our result reduces to the standard double-robustness result \citep{chernozhukov2018double}. Assumption 2 lifts the requirement for full combined error to go to 0 while Assumption 1 imposes restrictions on what constitutes a valid representation. 

\Cref{eqn:theorem_ovb} relates to results from the variable selection literature in causal inference \citep{belloni2014high,shortreed2017outcome,tang2023ultra,cho2023variable}. This literature is motivated by a classic result: to make valid causal inferences, it is sufficient to condition only on variables that affect both treatment assignment and outcome \citep{rosenbaum1983central}. Like Theorem 1, the results in this literature do not necessarily assume that the full outcome or propensity model can be consistently estimated as a function of the full set of covariates. While this literature proposes techniques when individual variables are shared in outcome and treatment models, these techniques do not apply when there is more complicated shared structure; for example, the number of years spent in a blue collar job may affect both treatment and outcome, but this is a transformation rather than a single variable. In contrast, our method is based on \textit{representations}, or potentially complicated functions of variables, rather than individual variables. A set of selected variables is an example of a representation; but representations can be more complex than a set of variables constructed by a researcher. 

Other methods from the causal inference and econometrics literature have also proposed using representations or latent variables from machine learning models. 
For example, \citet{battaglia2024inference} demonstrate that latent variables from a machine learning model should be jointly optimized with the econometric outcome of interest rather than first estimated separately and then plugged into an econometric model.
Related to our method, \citet{veitch2020adapting} provide a sufficient condition under which a representation is unbiased for estimating a causal effect, which motivates empirical methods used by \citet{shi2019adapting} and \citet{chernozhukov2022riesznet} (and is the basis of the multi-task debiased fine-tuning objective we consider). In contrast, we provide an if-and-only-if condition under which there is no bias, and we characterize the exact level of bias with a connection to omitted variable bias \citep{chernozhukov2022long}. Additionally, we provide conditions about the level of omitted variable bias under which estimation is $\sqrt{n}$-consistent and asymptotically normal. 

A strand of literature in supervised machine learning has also focused on integrating ideas from the causal inference literature into predictive methods in order to improve the properties of predictive models, such as stability (see \citet{cui2022stable} for a review of this literature). Similar to the approach in our paper, these methods adjust the training of a predictive model to avoid regularization-induced omitted variable bias, but the literature on stable prediction considers reducing such bias for many covariates simultaneously in a cross-sectional prediction problem. For example, some such methods reweight data to reduce the correlation among features.

\subsection*{Debiased Fine-Tuning}

The exact level of omitted variable bias (\Cref{eqn:ovb}) cannot be computed from data; it involves calculating the same high-dimensional function the representation is meant to approximate, $\mu_A(X)$. However, even if the bias cannot be computed exactly, we can still learn representations that are targeted to minimize it. Below, we develop three fine-tuning methods for minimizing this bias. Each method addresses omitted variable bias from a distinct angle. In principle, the optimal approach may vary across applications. For instance, multi-task fine-tuning is straightforward to implement but can require tuning an extra hyperparameter; projection fine-tuning removes that hyperparameter but may converge more slowly; difference-based fine-tuning can capture group disparities more directly but does not deliver a direct wage or propensity predictor. To choose among them, we recommend using validation metrics such as the R-Learner metric \citep{nie2021quasi} described (and subsequently used for model selection) in \Cref{sec:empirical_validation}.

\subsubsection*{Multi-Task Fine-Tuning}
The expression for omitted variable bias recalls a classic result from causal inference: to make valid causal inferences, it is sufficient to condition only on variables that affect both treatment assignment and outcome \citep{rosenbaum1983central}. Thus, if a representation captures all the features of history that are predictive of \textit{both} wage and group membership, it will result in zero omitted variable bias. While a representation that perfectly captures all of the variables related to wage (or equivalently to group membership) will result in zero omitted variable bias, it might be more realistic to capture the potentially smaller set of variables that affects both wage and group membership.

We therefore fine-tune the representations to be predictive of both wage and group membership. We consider two approaches. In \textbf{Multi-Task Fine-Tuning}, we use the method proposed by \citet{veitch2020adapting} and \citet{shi2019adapting} to learn a representation that jointly minimizes wage and group membership predictive errors:
\begin{align*}
    \textstyle \hat \lambda, \hat \mu, \hat e &= \textstyle \arg\min_{\lambda, \mu, \lambda} \big\{\sum_{i=1}^N \ell_Y[Y_i, \mu_{A_i}(\lambda(X_i))] + \beta*\ell_A[A_i, e(\lambda(X_i))]\big\},
\end{align*}
where $\beta\in\R^+$ is a hyperparameter, $\ell_Y$ is the mean squared-error loss, and $\ell_A$ is the binary cross-entropy loss. We also consider a similar approach, \textbf{Projection Fine-Tuning}, that alternates between losses: $\lambda$ is optimized to minimize mean-squared error loss until convergence, then $\lambda$ is optimized to minimize binary cross-entropy loss until convergence, and this process repeats until the procedure converges. This procedure is based on projected gradient descent \citep{calamai1987projected} and does not require choosing a hyperparameter $\beta$. See \Cref{sec:app:model_and_training_details} for more details.

\subsubsection*{Difference-Based Fine-Tuning}
An alternative debiased fine-tuning method is motivated by the heterogeneous treatment effect literature \citep{athey2016recursive,wager2018estimation,nie2021quasi}, where the goal is to model the difference in outcome functions for each group rather than each group's function individually. If the difference in these functions is simpler than each individual function, a method that's targeted to capture this difference will be more effective than modeling each function separately.

Thus, we propose a method for fine-tuning foundation models meant to capture group differences. This method is based on the R-learner approach for estimating heterogeneous treatment effects in causal inference \citep{robinson1988root,nie2021quasi}. It is based on the observation that group differences can be written as the solution to an objective:
\begin{equation}
\begin{split}
\arg\min_{\tau: \mathcal X \to \R} &\E\left\{\left[(Y - m(X)) - (A-e(X))\tau(X)\right]^2\right\} = \mu_1(X) - \mu_0(X),
\label{eqn:r_learner_objective}
\end{split}
\end{equation}
where $m(x) = \E[Y|X=x]$ is the conditional wage function \textit{averaged} over the two groups. 

The \textbf{Difference-Based Fine Tuning} procedure begins by estimating the conditional wage function $\hat m(\lambda_m(x))$ and the propensity function $\hat e(\lambda_e(x))$ using supervised fine-tuning. 
Starting with the wage, we estimate a function $\hat m$ and fine-tuned representation $\lambda_m$ in order to minimize the squared loss in predicting the wage $Y$. We estimate $\hat e$ and $\lambda_e$ analogously. Treating these functions and their respective representations as fixed, we then fine-tune a new representation to minimize the squared error:
\begin{equation*}
\hat \lambda, \hat \rho = \arg\min_{\lambda, \rho} \E\left\{\left[(Y - \hat m(\lambda_m(X))) - (A-\hat e(\lambda_e(X)))\rho( \lambda(X))\right]^2\right\},
\end{equation*}
where $\rho: \R^D \to \R$ is a flexible function; in the empirical studies we use a two-layer feed-forward neural network. The unexplained wage gap is then estimated as
\begin{equation}
\label{eqn:r_learner_outcome_model}
\textstyle \frac{1}{N_1} \sum_i A_i * \hat \rho(\hat\lambda(X_i)),
\end{equation}
where $N_1 = \sum_i A_i$. We refer to this method as \textit{difference-based fine-tuning}. 
This approach is based on the R-learner approach, but is based on representations: separate representations are used for the wage, propensity, and wage-difference models, so that the representation used for the last model is optimized to only capture the differences in groups. 
See \Cref{sec:app:model_and_training_details} for more details. 

This optimization procedure encourages a representation that captures differences in group wages. The true relationship between wage and labor market history may be complicated for both groups. However, if the \textit{difference} in relationships is not as complicated, it will be easier to learn a representation that captures the difference. 

\subsection*{Implementation details of fine-tuning} In practice, we implement each fine-tuning method by optimizing the respective objective using Adam \citep{kingma2014adam}. Each model is initialized at the foundation model's parameters and all parameters are re-optimized with respect to the new objective. The same neural architecture from pretraining (e.g. the transformer layers of CAREER) is retained, but we typically add extra parameters on top of the representation for our downstream tasks (e.g. predicting wages or propensities) and jointly update all parameters during fine-tuning. These steps enable the representation to adapt to the new objectives while preserving its broad, pretrained knowledge. See \Cref{sec:app:model_and_training_details} for more details. \section{Data}
\label{sec:data}
Our empirical analysis uses data from one of the leading U.S. administrative surveys, the Panel Study of Income Dynamics \citep{psid}, or PSID. PSID is a longitudinal survey that has followed a cohort of American families since 1968. It is constructed to be nationally representative and is frequently used to estimate unexplained wage gaps \citep{blau_us_2006,blau_gender_2017}. Because the same individuals are interviewed over the course of the survey, labor market histories can be constructed by tracking the trajectory of reported occupations each year an individual is in the survey. 

We encode occupations into one of 330 ``occ1990dd'' occupational categories \citep{david2013growth}. Since the PSID includes information about individuals who are not working, we add seven categories for when an individual's occupation is not listed but their employment status is available (e.g. employed, laid off). 

Following \citet{blau_gender_2017}, we restrict our sample to the surveys conducted between 1990-2019, consisting of 91,391 observations over 19 surveys, and further restrict our sample to non-farm and non-military wage and salary workers between 25 and 64 years old who worked for at least 26 weeks in non-farm jobs. We incorporate longitudinal sample weights into our analysis, which are designed to adjust for differences in the probability of selection into the sample. \Cref{app:data_preprocessing} contains more details about how we construct the dataset.

\section{Semi-Synthetic Experiments}
\label{sec:empirical_validation}
We first validate our proposed methods for debiased fine-tuning. 
Typically, machine learning models are evaluated using measures of predictive accuracy on held-out data.  However, our ultimate goal isn't forming more accurate wage predictions, but rather more accurate estimates of the unexplained wage gap. To this end, we turn to semi-synthetic experiments, a common method to assess causal estimation strategies in a controlled setting (e.g. \citet{athey2021using}). Semi-synthetic experiments allow us to evaluate the performance of different methods because the data generating process is known.

In order for the semi-synthetic experiments to reflect real-world data, we use real labor market histories from the PSID sample (Section \ref{sec:data}). Our semi-synthetic experiments are based on forming a ``ground-truth'' representation of the history and then generating group labels and wages as a function of this history. To mimic the fact that nature is often more complicated than the model we might select, we use a more complicated representation to generate data than the one used by the foundation model to estimate wage gaps. Specifically, we use a transformer architecture that is 20 times larger than the one used for estimation to simulate a ``ground-truth'' representation $\lambda^*: \mathcal X \to \R^D$. For each setting, we simulate group labels and wages as a function of the representation $\lambda^*$. We control how much of the representation is shared between these functions by introducing binary variables $\mathbf{u}, \mathbf{v} \in \{0, 1\}^D$ that mask the dimensions of the representation used for the group and wage models, accordingly. We then simulate from the following model:
\begin{align*}
    A_i &\sim \text{Bern}\left(\sigma\left(\textstyle \sum_j u_j\beta_j* \lambda(X_i)_j\right)\right)\\
    Y_i &= \textstyle \tau A_i + \sum_j v_j \beta_j  * \lambda(X_i)_j + \epsilon_i,
\end{align*}
where $\beta \in \R^D$ is a random vector of regression coefficients, $\tau \in \R$ is the true unexplained gap, $\epsilon_i \sim \N(0, \sigma^2)$ is the outcome noise, and $\sigma(\cdot)$ is the inverse-logit function. In this setup, $\mathbf{u}^\top \mathbf{v}/D$ is the proportion of the representation that is shared. For each experiment, we control the shared proportion ($\mathbf{u}^\top \mathbf{v}/D$), the true gap ($\tau$), and the level of outcome noise $(\sigma^2)$. We consider 27 different settings, and perform multiple samples in each setting by resampling $A_i$, $\epsilon_i$, $\mathbf{u}$, $\mathbf{v}$, and $\beta$. 
See \Cref{app:semi_synthetic_details} for more details. 

We compare four methods for estimating the unexplained wage gap from semi-synthetic data. Our baseline is \emph{Supervised Fine-Tuning} \citep{devlin2018bert}: fine-tuning a foundation model to predict wage without an explicit debiasing objective. We compare this baseline to the three debiasing methods described above: Multi-Task Fine-Tuning, Projection Fine-Tuning, and Difference-Based Fine-Tuning. We perform 400 simulations, re-generating data and re-training each method for each simulation. Here we estimate the unexplained gap using the outcome-only estimator (analogous to \Cref{eqn:outcome_only_model}), while \Cref{app:semi_synthetic_details} shows results for the AIPW estimator (\Cref{eqn:aipw}), which we find to perform slightly worse in practice. 
Although AIPW is theoretically advantageous in large samples  \citep{robins1997toward,tsiatis2006semiparametric}, it can suffer from higher variance when propensity models are difficult to estimate, thereby degrading its finite-sample performance \citep{kang2007demystifying}.

\begin{figure*}
\centering
\includegraphics[width=1.0\linewidth]{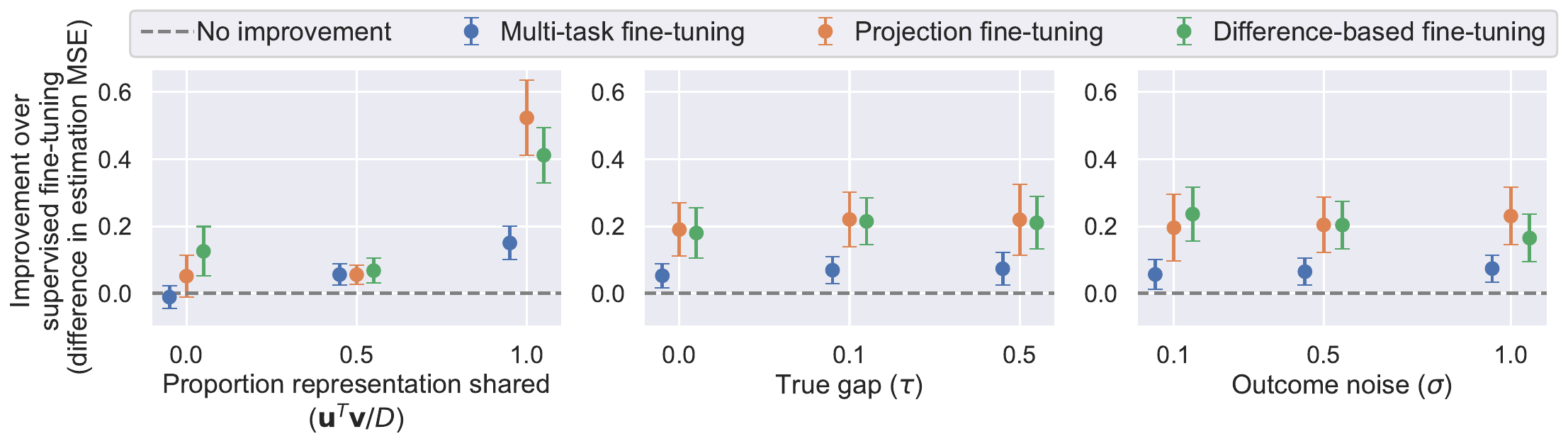}
\caption{
Debiased fine-tuning methods are better at estimating the unexplained wage gap than standard supervised fine-tuning \citep{devlin2018bert} across 270 semi-synthetic experiments. For each semi-synthetic experiment, the true unexplained wage gap is known, and each method provides a different estimate of this gap. This figure compares each method's average error for estimating this gap, evaluated via MSE between the true and estimated unexplained gap. Specifically, the Y-axis compares each method's estimation error to the error from estimates derived from a model using standard supervised fine-tuning (larger values on the Y-axis correspond to larger improvements). Bars represent 95\% confidence intervals. 
}
\label{fig:semi_synthetic_experiments}
\end{figure*}

\Cref{fig:semi_synthetic_experiments} compares the MSE of the estimate of the gender wage gap (relative to the oracle gender wage gap) derived from alternative estimation approaches. All three methods for debiasing foundation models consistently outperform the standard supervised fine-tuning approach. The advantage of debiasing is largest when more of the representation is shared across the wage model and group labels. This reinforces the motivation behind representation learning; as there is more shared structure in how group labels and wages relate to history, sharing representations can improve estimates. 
Projection fine-tuning and difference-based fine-tuning are both more successful than multi-task fine-tuning, especially when more of the representation is shared. The full set of results is in \Cref{tab:app:all_semi_synthetic_results}.

How should we validate models on real-world data? While wage and gender predictive metrics are important, they do not directly assess estimation quality. While matching-based methods \citep{stuart2010matching} can also be used to validate estimation in principle, these require low-dimensional covariates. 
We instead consider another validation metric, inspired by the R-Learner objective in \Cref{eqn:r_learner_objective}. 
Because \Cref{eqn:r_learner_objective} is minimized when $\tau(X)$ is the true expected difference female and male wages, we evaluate \Cref{eqn:r_learner_objective} given a model's estimate of $\hat \tau(X) = \hat \mu_1(X) - \hat \mu_0(X)$, 
\begin{equation}
    \textstyle \frac{1}{n}\sum_{i=1}^n [(Y_i - \hat m(X_i)) - (A_i - \hat e(X_i)) \hat \tau(X_i)]^2,
    \label{eqn:r_learner_eval_metric}
\end{equation}
where $i=1, \dots, n$ index $n$ held-out samples and $\hat m(X_i)$ and $\hat e(X_i)$ are models trained to predict wage and gender, respectively, using supervised fine-tuning of CAREER. We refer to \Cref{eqn:r_learner_eval_metric} as the \textbf{R-Learner Metric}. 
We note that because this metric relies on estimates of wage and propensity models, it is sensitive to the specification of $\hat m$ and $\hat e$. However, we validate this metric on semi-synthetic data, finding it to be a useful proxy for model performance (see \Cref{app:fig:semi_synthetic_evaluation_metrics}).

 \section{Empirical Application}

We now apply our methods to the (actual) PSID data.  We begin by evaluating the quality of wage predictions derived from alternative models, including the standard econometric models from the wage gap literature \citep{blau_gender_2017}, since predictive accuracy can be easily evaluated using held-out test data. We show that foundation-based representations substantially improve predictive performance relative to standard regression-based econometric models, suggesting that our methods can capture variables that have the potential to cause omitted variable bias when estimating unexplained wage gaps. We then directly demonstrate that representations derived from our methods capture elements of history that are predictive of \emph{both} wage and gender, so that they indeed meet the criteria for omitted variable bias, and that these are quantitatively important for explaining wage gaps.

\subsection{Predictive accuracy}
\label{sec:predictivePSID}
As baselines, we consider econometric models that use hand-constructed summaries of an individual's career but not their full history to predict wage. Following the econometric literature, we consider two linear models: regression (fit with OLS) and LASSO. Given covariates $Z_i \in \R^P$, these models estimate the wage function as 
\begin{equation}
    \label{eqn:linear_model}
    \hat \mu_A(Z_i) = \theta_A + \beta_A^\top Z_i,
\end{equation}
for $A \in \{0, 1\}$, an intercept $\theta_A \in \R$, and regression coefficients $\beta_A \in \R^P$. Following \citet{blau_gender_2017}, the covariates included in $Z$ are: years of full-time and part-time experience (and their squares), years of schooling, indicators for bachelors and advanced degrees, race and ethnicity indicators, gender indicators, census and region indicators, an indicator for collective bargaining coverage, 15 industry category indicators, and 21 occupation category indicators. We also consider two different methods of encoding occupations: \textbf{``coarse-grained''}, which uses the 21 coarse-grained occupational categories above, and \textbf{``fine-grained''}, with an additional 330 fine-grained occupational categories. 

We compare these models from the economics literature to predictions based on foundation models that use CAREER \citep{vafa2023career} to represent labor market history. CAREER is pretrained to learn representations of career trajectories on a dataset of 23.7 million resumes. 
We consider both supervised fine-tuning \citep{devlin2018bert} and debiased fine-tuning approaches to modify CAREER's representations. 
When we fine-tune CAREER, we use both its representations of history and the covariates $Z_i$ described above to predict an individual's wage; see \Cref{sec:app:model_and_training_details} for more details. In order to understand how different methods of including history affect predictions, we train two additional versions of CAREER: one that uses the neural network to encode an individual's current job but not their history (\text{``CAREER (current job only)''}), and one that includes an individual's current job but only their workforce participation status for previous jobs (e.g. ``unemployed'', ``out-of-labor force''), which we refer to as ``\text{CAREER (participation only)}''.

\begin{table*}
  \footnotesize
  \centering
  \caption{The CAREER foundation model, which incorporates an individual's full labor market history, forms better predictions of wage and gender on held-out data than standard econometric methods, which summarize history with low-dimensional statistics. Test-set bootstrapped standard errors are in parentheses.
  }
  \begin{tabular}{c l c c}
    \toprule
  & & Wage $R^2$ & Gender $R^2$ \\
  \midrule 
  \textbf{Regression models} & Coarse-grained regression & 0.428 (0.004) &  0.137 (0.002) \\
   & Coarse-grained LASSO & 0.428 (0.003) &  0.260 (0.004) \\
  & Fine-grained LASSO & 0.455 (0.003) &  0.314 (0.004) \\
  \midrule
  \textbf{Foundation models} & CAREER (no pretraining) & 0.462 (0.003) & 0.424 (0.005) \\
  \textbf{(supervised fine-tuning)} & CAREER (pretrained, current job only) & 0.454 (0.004) &  0.307 (0.003) \\
   & CAREER (pretrained, participation only) & 0.467 (0.003) &  0.309 (0.003) \\
  & CAREER (pretrained) & 0.515 (0.004) & 0.510 (0.003)\\
  \midrule 
  \textbf{Foundation models} & CAREER (multitask fine-tuning) & 0.468 (0.004) &  0.514 (0.004) \\
  \textbf{(debiased fine-tuning)} & CAREER (projection fine-tuning) & 0.503 (0.003) & 0.513 (0.004) \\
  \bottomrule
 \end{tabular}
  \label{tab:predictive_accuracy_full}
\end{table*}

Table \ref{tab:predictive_accuracy_full} shows the held-out $R^2$ for each model's wage and gender predictions. (Since gender isn't real-valued, we use pseudo $R^2$ based on negative log-likelihood.)
CAREER outperforms all the econometric baselines. 
With the standard supervised fine-tuning, CAREER has a held-out wage $R^2$ of 0.515 and held-out gender $R^2$ of $0.510$. Its predictive performance is not stemming from including a better functional form for an individual's current job or capturing employment spells more fully. 
For the two debiased fine-tuning methods, we find that wage $R^2$ slightly worsens, while gender predictions might slightly improve, although this improvement is within the standard errors. (We do not have wage or gender predictions for difference-based fine-tuning because it predicts the difference in group wages rather than individual wages or genders.)  
These results extend a finding from \citet{vafa2023career}, which also shows that transformer-based methods can improve wage predictions relative to econometric baselines. What \Cref{tab:predictive_accuracy_full} additionally shows is that gender predictions are improved by a larger margin and that these gains are still present for debiased fine-tuning methods. 
This result demonstrates another benefit of representation learning; if group membership and wage are correlated with similar transformations of input data, then learning representations that are predictive of both can improve predictions. 

\subsection{Analyzing the Gender Wage Gap}
\label{sec:gapanalysis}
We now compare gender wage gaps estimated with standard econometric techniques to those estimated with a foundation model's representations of labor market history, following the approaches described above.
For the econometric models, we use a linear regression using the same covariates described in \Cref{sec:predictivePSID}, encoding occupation into one of 21 coarse-grained labels.  
For the foundation models, we fine-tune CAREER using each of the three  debiased fine-tuning methods described in \Cref{sec:empirical_validation}. 
In addition to using the machine learned representations of history, these methods also incorporate the same hand-constructed covariates as the linear model. 
Because the unexplained wage gap is only identified when there is overlap (e.g. when there are workers with similar histories in both groups), we trim the study population. Specifically, we fine-tune CAREER with supervised fine-tuning to estimate a propensity model $\hat e(\lambda(X_i))$, and only include individuals $i$ such that $0.01 < \hat e(\lambda(X_i)) < 0.99$. 
We consider other trimming strategies in \Cref{app:sec:additional_tables}, finding similar results. We compute standard errors with bootstrapping, where the standard errors reflect sampling uncertainty conditional on the fine-tuned model. We do not re-train models for each bootstrap sample. Instead, we keep models fixed, evaluating each model on the bootstrapped sample; we refer to this as \textit{test-set bootstrapping}.

The results are summarized in \Cref{tab:wage_gaps}. 
Across all fine-tuning methods, full history explains more of the gender wage gap than hand-constructed summaries of history, which are typically used to explain gender wage gaps \citep{blau_trends_2013,goldin_grand_2014,blau_gender_2017}. While the wage ratio unexplained by summaries of history is 88.6\%, the ratio is above 90\% for all methods that use representations of full history, ranging from 90.4\% for projection fine-tuning to 93.4\% to difference-based fine-tuning. 
Note that these numbers vary a little depending on the trimming threshold and whether the AIPW estimator (\Cref{eqn:aipw}) is used instead of the outcome-only estimator; additional results are presented in \Cref{app:sec:additional_tables}. 
\Cref{tab:wage_gaps} also shows the R-Learner metric in \Cref{eqn:r_learner_eval_metric} for each model, which is minimized by difference-based fine-tuning. The last column of \Cref{tab:wage_gaps} shows that the difference-based fine-tuning method's R-Learner metric improvement over the the regression and multi-task fine-tuning improvement is significant at the 95\% level. 
Our analysis for the remainder of the paper considers results from the difference-based fine-tuning approach.

To further investigate where history is explaining the gender wage gap, we consider different cuts of the survey data. 
\Cref{app:fig:wage_ratios_over_time} shows the gender wage ratios over time. Compared to the methods that adjust the wage gap for summary statistics of history, the learned representations of history are explaining the least of the gap in 1990-1995 and the most of the gap in 2014-2019. 
Overall, we find that compared to methods that adjust the gap for summary statistics of history, full history explains more of the gap for later years. 
While this may reflect changes in the underlying wage dynamics between these two periods, it may also reflect the fact that more history is available in the later versions of the survey. To further investigate, \Cref{app:fig:wage_ratios_by_age} fixes the time period to 2014-2019 and considers different wage ratios by age. For the youngest workers, history explains the least of the gap; the history-explained ratio is almost the same as the unadjusted ratio for 25-34 year-old workers. However, history consistently explains more of the gap as workers become older. Further, \Cref{app:fig:wage_ratios_by_occupation} breaks down the adjusted wage gap by the occupational categories considered in \citet{blau_gender_2017}. Compared to the wage gap explained by summary statistics of history, the wage gap explained by representations of full history is smoother across occupational categories. The occupational category where history explains the most of the wage gap is in non-physician healthcare-related occupations; the category where it explains the least of the wage gap is for computer-related occupations.

\begin{table*}
  \footnotesize
  \centering
  \caption{Machine-learned representations of career history explain more of the gender wage gap than hand-constructed summary statistics typically used to estimate the gap. These representations are also validated by better performance on the R-Learner metric (\Cref{eqn:r_learner_eval_metric}). All results are estimated with cross-fitting on 5-folds, using the outcome-only estimator. Data is for PSID from 1990-2019 with 1\% clipping. Test-set bootstrapped standard errors are in parentheses.
  }
    \resizebox{\textwidth}{!}{
\begin{tabular}{l l c | c c}
& & & & \multicolumn{1}{c}{\textbf{R-Learner metric relative}} \\
& \textbf{Adjustment method} & \textbf{Wage ratio}  & \textbf{R-Learner metric} & \multicolumn{1}{c}{\textbf{to difference-based fine-tuning}}\\
\toprule
  Unadjusted & --- & 0.775 (0.004) & --- & --- \\
  \midrule
  Adjusted for summary statistics & Linear regression & 0.886 (0.001) &  0.1841 (0.0018)  &-0.00122 (0.00029) \\
  \midrule 
  Adjusted for full history & Multi-task fine-tuning & 0.907 (0.001)  & 0.1835 (0.0018) &  -0.00059 (0.00027) \\
  &  Projection fine-tuning & 0.904 (0.001) & 0.1833 (0.0018) & -0.00036 (0.00025) \\
  & Difference-based fine-tuning & 0.934 (0.001) &  0.1829 (0.0019) & --- \\
  \bottomrule
 \end{tabular}}
  \label{tab:wage_gaps}
\end{table*}

\begin{figure*}
\centering
\includegraphics[width=1.0\linewidth]{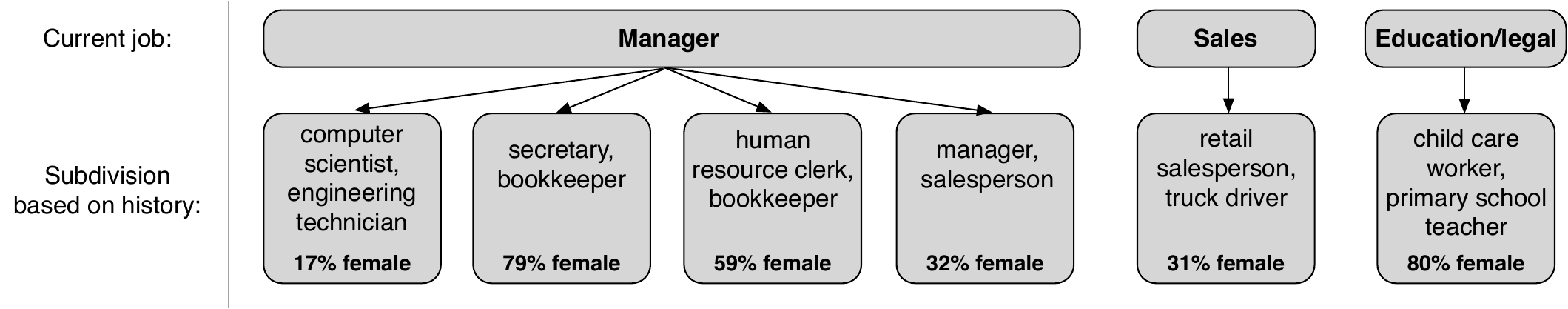}
\caption{CAREER finds omitted variables from a worker's job history that are important for explaining the gender wage gap. These omitted variables, which are identified by a regression tree as being most predictive of wage, are correlated with both wage and gender.
}
\label{fig:history_clusters}
\end{figure*}

The promise of incorporating more complete representations of history into wage gap analyses is that they can include variables that are typically omitted (recall from \Cref{sec:framework} that omitted variables include variables that are correlated with both wage and gender). This is not unique to machine learning methods; for example, prior studies have found that potential experience (an inexact measure of experience that does not measure workforce interruptions) explains less of the wage gap than years of actual work experience \citep{regan2009work,blau2013feasibility}, as do representations of history that do not include timing of labor force participation \citep{light1995early}. What machine learning  offers is the ability to automatically learn these variables without prespecifying them. 

Here, we investigate the aspects of history that are captured by the foundation model but omitted by traditional methods that summarize history with summary statistics. While our model learns representations of history for each individual, they are difficult to interpret directly because they are continuous. To better understand these representations of histories, we form clusters of histories. We constrain each history in a cluster to have the same current occupational category but allow the kinds of history in a cluster to vary. We then study which clusters are important for predicting wages by building a regression tree to predict wage from the clusters. The regression tree adds clusters one-at-a-time in the order that is most useful for minimizing wage error. This ordering allows us to identify and interpret which occupation patterns --- beyond an individual's current job --- are most predictive of wage.  

\Cref{fig:history_clusters} shows the six most predictive groups (see \Cref{app:sec:qualitative_details} for more details and \Cref{tab:app:all_15_top_clusters} for the top 15 groups).
These groups reveal important aspects of history that are omitted by hand-constructed summaries. For example, one of the most predictive types of histories consists of managers who were previously computer scientists and engineering technicians. This is a predominantly male group (83\% male), and managers with these jobs in their histories get paid more than managers without them. There are multiple interpretations for why managers with these histories are paid more than managers without them; perhaps these managers have different skills than other managers, or perhaps they are performing different jobs that are not captured by the occupational encoding scheme (e.g. there is no occupational category for engineering manager). In addition to this group, the other top groups that are important for wage predictions are also correlated with gender. 
Omitting these variables --- like standard econometric methods do --- induces omitted variable bias. \Cref{tab:app:fine_grained_omitted_variables} includes additional analyses of omitted variables found at the fine-grained occupational level. 

 \section{Discussion}

We used foundation models to study a classic problem from labor economics: 
estimating the difference between how individuals with the same labor market experience get paid when they belong to different groups. 
With foundation models, wage predictions improve over econometric baselines by up to 15\%. We also showed that an omitted variable bias arises when a foundation model discards relevant information about group differences. To mitigate this problem, we proposed procedures for debiasing foundation models, which we validated on semi-synthetic data. On survey data from the Panel Study of Income Dynamics, we found that labor market history explains more of the gender wage gap than the summary statistics of history used by standard econometric methods. 

These findings are suggestive of how to use foundation models in social science research. One direct application is using the debiasing methods we propose to estimate causal effects with foundation models. While we study a foundation model of labor market history trained on resume data, these methods can extend to analyses involving other foundation models. For example, foundation models trained on rich, nationwide administrative data can help answer a variety of descriptive and causal questions \citep{savcisens2024using}. 

Additionally, our methods can help address questions about the representativeness of large language models (LLMs), the most common type of foundation model. 
A recent literature has found that large language models, when queried to answer questions from surveys, do not respond in ways that are representative of the national population \citep{santurkar2023whose}. The problem of representative predictions and debiased foundation models are closely relate. Our methods could be adapted to improve the representativeness of LLMs. 

\paragraph{Acknowledgements.}
Keyon Vafa is supported by the Harvard Data Science Initiative. 
Susan Athey is supported by the Golub Capital Social Impact Lab, the Stanford Institute for Human-Centered Artificial Intelligence, and the Business, Governance, and Society Initiative at Stanford’s Graduate School of Business.
David M. Blei is supported by NSF IIS-2127869, NSF DMS-2311108, ONR N000142412243, and the Simons Foundation. 
\bibliography{labor}
\bibliographystyle{bibstyle}

\clearpage
\appendix

\section{OVB Derivation}
\label{app:ovb_derivation}
We now derive \Cref{eqn:ovb}. We write out the expression:
\begin{align}
    &\int (\mu_a(\lambda(x))-\mu_a(x))\left(\frac{1-a}{P(A=1)}\right)\left(\frac{e(\lambda(x))}{1-e(\lambda(x))}-\frac{e(x)}{1-e(x)}\right) P(x,a)dxda \nonumber \\
    &= \int (\mu_0(\lambda(x))-\mu_0(x))\left(\frac{P(A=0)}{P(A=1)}\right)\left(\frac{e(\lambda(x))}{1-e(\lambda(x))}-\frac{e(x)}{1-e(x)}\right) P(x|A=0)dx. \label{eqn:app:ovb}\\
\end{align}

First study the term involving $e(\lambda(x))$:
\begin{align}
    &\int (\mu_0(\lambda(x))-\mu_0(x))\left(\frac{P(A=0)}{P(A=1)}\right)\left(\frac{e(\lambda(x))}{1-e(\lambda(x))}\right) P(x|A=0)dx \nonumber \\
    &= \E\left[ (\mu_0(\lambda(X))-\mu_0(X))\left(\frac{P(A=0)}{P(A=1)}\right)\left(\frac{e(\lambda(X))}{1-e(\lambda(X))}\right) \bigg|  A=0\right]  \nonumber \\
    &= \E\left[\E\left[ (\mu_0(\lambda(X))-\mu_0(X))\left(\frac{P(A=0)}{P(A=1)}\right)\left(\frac{e(\lambda(X))}{1-e(\lambda(X))}\right) \bigg| \lambda(X), A=0\right] \bigg| A=0\right] \label{eqn:app:cond_exp}.
\end{align}
Note that
\begin{align*}
    \E[\mu_0(X)|\lambda(X), A=0] &= \E[\E[Y|X, A=0]|\lambda(X), A=0]\\
    &= \E[\E[Y|X,\lambda(X), A=0]|\lambda(X), A=0]\\
    &= \E[Y|\lambda(X), A=0]\\
    &= \mu_0(\lambda(X)),
\end{align*}
where the second equality is because $\lambda(X)$ is a deterministic function of $X$ and the third equality follows from the law of iterated expectation. Therefore, \Cref{eqn:app:cond_exp} reduces to
\begin{align*}
       &\E\left[\E\left[ (\mu_0(\lambda(X))-\mu_0(X))\left(\frac{P(A=0)}{P(A=1)}\right)\left(\frac{e(\lambda(X))}{1-e(\lambda(X))}\right) \bigg| \lambda(X), A=0\right] \bigg| A=0\right] \\
       &= \E\left[(\mu_0(\lambda(X))-\mu_0(\lambda(X)))\left(\frac{P(A=0)}{P(A=1)}\right)\left(\frac{e(\lambda(X))}{1-e(\lambda(X))}\right) \bigg|  A=0\right]\\
       &= 0
\end{align*}
Now study the remainder of \Cref{eqn:app:ovb}. 
\begin{align*}
    &\int (\mu_0(\lambda(x))-\mu_0(x))\left(\frac{P(A=0)}{P(A=1)}\right)\left(-\frac{e(x)}{1-e(x)}\right) P(x|A=0)dx \nonumber \\
   &= \int (\mu_0(x)-\mu_0(\lambda(x)))\left(\frac{P(A=0)}{P(A=1)}\right)\left(\frac{e(x)}{1-e(x)}\right) P(x|A=0)dx \nonumber \\
   &= \int (\mu_0(x)-\mu_0(\lambda(x)))\left(\frac{P(A=0)}{P(A=1)}\right)\left(\frac{P(A=1|x)}{P(A=0|x)}\right) P(x|A=0)dx \nonumber \\
   &= \int (\mu_0(x)-\mu_0(\lambda(x)))\left(\frac{P(A=0)}{P(A=1)}\right)\left(\frac{P(A=1,x)}{P(A=0,x)}\right) \frac{P(x,A=0)}{P(A=0)}dx \nonumber  \\
   &= \int (\mu_0(x)-\mu_0(\lambda(x))) P(x|A=1) dx \nonumber \\
   &= \int (\mu_1(\lambda(x))-\mu_0(\lambda(x))) P(x|A=1) dx - \int (\mu_1(x) - \mu_0(x)) P(x|A=1) dx \nonumber \\
   &= \text{OVB}(\lambda)
\end{align*}

\section{Proof of Theorem 1}
\label{app:proof_of_theorem}
Consider random variables $(X, A, Y) \sim P$, where $X \in \R^D, A \in \{0, 1\}$, and $Y \in \R$. We'll use $Z = (X, A, Y)$ as a shorthand to denote all the variables. 
A representation is a function: $\lambda: \R^D \to \R^K$. 
We'll use $P_n$ to denote the empirical measure and write sample averages as $P_n\{f(Z)\} = \frac{1}{n}\sum_i f(Z_i)$. In general we'll write expectations as $P\{f(Z)\} = \int f(z)dP(z)$. 

Consider a quantity mathematically identical to the average treatment effect (ATE), which we write out as a function of the distribution, $\psi(P)$:
\begin{equation}
\label{eqn:estimand}
\psi(P) = \E_{P}\{\E_{P}[Y|X, A=1] - \E_{P}[Y|X, A=0]\}.
\end{equation}
Theorem 1 is about a quantity analogous to the average treatment effect on the treated (ATT) rather than the ATE. However the proof with the ATT follows from the proof with the ATE so we write out the proof with the ATE for generality. 

Define the representation-based estimand: 
\begin{equation}
\label{eqn:representation_estimand}
\psi_\lambda(P) = \E_{P}\{\E_{P}[Y|\lambda(X), A=1] - \E_{P}[Y|\lambda(X), A=0]\}.
\end{equation}

Slightly abusing notation, define the following quantities:
\begin{alignat*}{2}
\mu_a(x) &= \E_P[Y|X=x,A=a], &\qquad \mu_a(\lambda(x)) &= \E_P[Y|X=\lambda(x), A=a] \\
e(x) &= P(A=1|X=x), &\qquad e(\lambda(x)) &= P(A=1|\lambda(X)=\lambda(x)) \\
\alpha_a(x) &= \textstyle \left(\frac{a}{e(x)}-\frac{1-a}{1-e(x)}\right), &\qquad \alpha_a(\lambda(x)) &= \textstyle \left(\frac{a}{e(\lambda(x))}-\frac{1-a}{1-e(\lambda(x))}\right).
\end{alignat*}

Define the \text{influence function} $\varphi_\lambda(Z;P)$ for a fixed representation $\lambda$:
\begin{equation*}
\varphi_\lambda(Z;P) = \textstyle \left(\frac{A}{e(\lambda(X))}-\frac{1-A}{1-e(\lambda(X))}\right)(Y-\mu_A(\lambda(X)))+ \mu_1(\lambda(X)) - \mu_0(\lambda(X)) - \psi_\lambda(P). 
\end{equation*}

Consider an estimator of $\psi(P)$ constructed from $n$ i.i.d. samples from $P$. This estimator relies on estimating a representation $\lambda_n$ and nuisance functions for $\mu_a(\lambda_n(x))$ and $e(\lambda_n(x))$ 
denoted by $\hat \mu_{a,n}(\lambda_n(x))$ and $\hat e_n(\lambda_n(x))$.
We'll use $\hat P_{n, \lambda_n}$ to refer to a probability distribution consistent with $\hat \mu_{n,a}(\lambda_n(x))$, $\hat e_n(\lambda_n(x))$, and the empirical measure over $\lambda_n(x)$. Then we define an estimator of the ATE as follows:
\begin{equation}
\hat \psi = \psi_{\lambda_n}(\hat P_{n,\lambda_n}) + P_n\{\varphi_{\lambda_n}(Z;\hat P_{n,\lambda_n})\}
\end{equation}
where by construction 
\begin{align*}
\psi_{\lambda_n}(\hat P_{n,\lambda_n}) = & \textstyle P_n\{\hat \mu_{n,1}(\lambda_n(X_i)) - \hat \mu_{n,0}(\lambda_n(X_i))\} = 
\frac{1}{n} \sum_{i=1}^n \left(\hat \mu_{n,1}(\lambda_n(x_i)) - \hat \mu_{n,0}(\lambda_n(x_i))\right)\\
\varphi_{\lambda}(Z; \hat P_{n,\lambda_n}) = & \textstyle \left(\frac{A}{\hat e_n(\lambda_n(X))}-\frac{1-A}{1-\hat e_n(\lambda_n(X))}\right)(Y-\hat\mu_{n,A}(\lambda_n(X))) \\
&+ \hat\mu_{n,1}(\lambda_n(X)) - \hat\mu_{n,0}(\lambda_n(X)) - \psi(\hat P_{n,\lambda_n}).
\end{align*}

\begin{theorem}
Assume the following:
\begin{enumerate}
  \item \textbf{Omitted variable bias goes to 0 at a root-n rate}:
  \begin{equation}
    \label{assumption:ovb}
    \text{Cov}_P\left(\mu_A(X) - \mu_A(\lambda_n(X)),\alpha_A(X) - \alpha_A(\lambda_n(X))\right) = o_P(n^{-1/2})
  \end{equation}
  \item \textbf{Combined root-n consistency as a function of representation}:
  \begin{equation}
    \label{assumption:combined_root_n_error}
    \|\hat e_n(\lambda_n(X)) - e(\lambda_n(X))\| \|\hat \mu_{n,a}(\lambda_n(X))-\mu_{n,a}(\lambda_n(X))\| = o_P(n^{-1/2}),
  \end{equation}
  for $a \in \{0, 1\}$. 
  \item \textbf{Cross-fitting}: $\hat \mu_n$, $\hat e_n$, and $\lambda_n$ are estimated on a different sample than the $n$ samples used to construct $\hat \psi$. 
  \item \textbf{Convergence of the representation}: There exists a representation $\lambda^*$ such that 
  \begin{equation}
    \label{assumption:convergence_of_representation}
    P_n\{\varphi_{\lambda_n}(Z;P) - \varphi_{\lambda^*}(Z;P)\} = o_P(n^{-1/2}),
  \end{equation}
  and $\text{Var}(\varphi_{\lambda^*}(Z;P)) < \infty$.
  \item \textbf{Consistency as a function of representations}:
  \begin{equation}
      \|\hat e_n(\lambda_n(X)) - e(\lambda_n(X))\| = o_P(1) \ \ \text{ and } \ \ \|\hat \mu_{n,a}(\lambda_{n,a}(X)) - \mu_{n,a}(\lambda_n(X))\| = o_P(1),
  \end{equation}
  for $a \in \{0, 1\}$. 
  \item \textbf{Strict overlap}: There exists a value $\epsilon$ with $0 < \epsilon < 0.50$ such that $\epsilon \leq e(X) \leq 1-\epsilon$, $\epsilon \leq e(\lambda_n(X)) \leq 1-\epsilon$, and $\epsilon \leq \hat e(\lambda_n(X)) \leq 1-\epsilon$ with probability 1. 
  \item \textbf{Boundedness}: There exists a constant $C < \infty$ such that $|Y-\hat \mu_{n,a}(\lambda_n(X))| \leq C$ for $a \in \{0, 1\}$ with probability 1.
\end{enumerate}
Then
\begin{equation}
\sqrt{n}(\hat \psi - \psi(P)) \to \N(0, \text{Var}(\varphi_{P_{\lambda^*}}(Z;P))).
\end{equation}
\end{theorem}

\begin{proof}
We want to analyze the following difference:
\begin{equation}
\hat \psi - \psi(P).
\end{equation}
Note that we have the following decomposition:
\begin{equation}
\label{eqn:ovb_decomposition}
\hat \psi - \psi(P) = \underbrace{\hat \psi - \psi_{\lambda_n}(P)}_{A_n} + \underbrace{\psi_{\lambda_n}(P) - \psi(P)}_{B_n}.
\end{equation}
The first difference $A_n$ is the difference between the true and estimated functional \textit{for a fixed representation}. The second difference $B_n$ does not involve estimated functions; it is the difference in the target quantity for the representation and for the full $X$. In other words, $B_n$ is the omitted variable bias. Our proof will proceed by using standard DML proof techniques to show that $\sqrt{n}A_n$ converges in probability to a zero-mean Gaussian distribution and use Assumption 1 to show that $B_n$ goes to 0 at a root-n rate. 

We first analyze $A_n$. Our technique will follow \citep{kennedy2022semiparametric}. We have 
\begin{equation}
\label{eqn:error_expansion}
\hat \psi - \psi_{\lambda_n}(P) = \psi_{\lambda_n}(\hat P_{n,\lambda_n}) -  \psi_{\lambda_n}(P) + P_n\{\varphi_\lambda(Z;\hat P_{n,\lambda_n})\}
\end{equation}
Note that we have the following von Mises expansion for any two distributions $P$ and $\hat P$:
\begin{align*}
\textstyle \psi_{\lambda}(\hat P) - \psi_{\lambda}(P) &= \textstyle \int \varphi_\lambda(z;\hat P)d(\hat P - P)(z) + R_2(\hat P, P) \\
 &=\hat P\{\varphi_\lambda(Z;\hat P)\} - P\{\varphi_\lambda(Z;\hat P)\} + R_2(\hat P, P)\\
 &=-P\{\varphi_\lambda(Z;\hat P)\} + R_2(\hat P, P)
\end{align*}
where
\begin{align*}
R_2(\hat P, P) &= \textstyle \int \left(\frac{1}{\hat e(\lambda(x))} - \frac{1}{e(\lambda(x))}\right)(\mu_1(\lambda(x)) - \hat \mu_1(\lambda(x)))e(\lambda(x))dP(x) \\
& \ \ \ \ + \textstyle \int \left(\frac{1}{1-\hat e(\lambda(x))} - \frac{1}{1-e(\lambda(x))}\right)(\hat\mu_0(\lambda(x)) - \mu_0(\lambda(x)))(1-e(\lambda(x)))dP(x).
\end{align*}
and the last equality follows from
\begin{equation}
\label{eqn:von_mises_zero_mean}
\hat P\{\varphi_\lambda(Z;\hat P)\} = 0.
\end{equation}
(These can be verified with basic calculus.) Then we can write \Cref{eqn:error_expansion} as 
\begin{align*}
\hat \psi - \psi_{\lambda_n}(P) &= (P_n-P)\{\varphi_{\lambda_n}(Z;\hat P_{n, \lambda_n})\} + R_2(\hat P_{n,\lambda_n}, P)\\
&= (P_n-P)\{\varphi_{\lambda_n}(Z;P)\} + (P_n-P)\{ \varphi_{\lambda_n}(Z; \hat P_{n,\lambda_n}) - \varphi_{\lambda_n}(Z;P) \} \\
& \ \ \ \ \ \ \ \ + R_2(\hat P_{n, \lambda_n}, P)\\
& \equiv S^* + T_1 + T_2
\end{align*} 
We will show that $\sqrt{n}S^*$ behaves asymptotically like a Gaussian up to $o(1)$ error, and then show that $T_1$ and $T_2$ are each $o(n^{-1/2})$. 

We first analyze the first term:
\begin{align*}
S^* &= (P_n-P)\{\varphi_{\lambda_n}(Z;P)\} \\
&= P_n\{\varphi_{\lambda_n}(Z;P)\}\\
&= P_n\{\varphi_{\lambda^*}(Z;P)\} + P_n\{\varphi_{\lambda_n}(Z;P) - \varphi_{\lambda^*}(Z;P)\}
\end{align*}
where the second equality comes from the fact the influence function is zero-mean (\Cref{eqn:von_mises_zero_mean}) and the last equality comes from adding and subtracting the same term. By Assumption 4 we have
\begin{equation}
\label{eqn:sample_average_assumption}
P_n\{\varphi_{\lambda_n}(Z;P) - \varphi_{\lambda^*}(Z;P)\} = o_P(n^{-1/2}),
\end{equation}
so $S^* = P_n\{\varphi_{\lambda^*}(Z;P)\} + o_P(n^{-1/2})$. This is a sample average of a zero-mean function plus an $o_P(n^{-1/2})$ term, so by the Central Limit Theorem $\sqrt{n}S^*$ converges in distribution to a zero-mean normal distribution with variance $\text{Var}(\varphi_{\lambda^*}(Z;P))$. Therefore,
\begin{equation}
\hat \psi - \psi(P) = P_n\{\varphi_{\lambda^*}(Z; P)\} + T_1 + T_2 + B_n + o_P(n^{-1/2}).
\end{equation}

We now analyze the second term $T_1$:
\begin{align*}
T_1 &= (P_n-P)\{ \varphi_{\lambda_n}(Z; \hat P_{n,\lambda_n}) - \varphi_{\lambda_n}(Z;P) \}.
\end{align*}
Using the definition of the influence function, we can write out
\begin{align*}
T_1 &= (P_n-P)\{ \varphi_{\lambda_n}(Z; \hat P_{n,\lambda_n}) - \varphi_{\lambda_n}(Z;P)\}\\
&= \textstyle (P_n-P)\bigg\{\left(\frac{A}{\hat e(\lambda_n(X))}-\frac{1-A}{1-\hat e(\lambda_n(X))}\right)(Y-\hat \mu_A(\lambda_n(X))) + \hat \mu_1(\lambda_n(X))-\hat \mu_0(\lambda_n(X)) -\psi_{\lambda_n}(\hat P_{n,\lambda_n})\\
& \ \ \ \ \ \ \ \ \  \ \ \ \ \ \ \ \ \ \ \ \  \ \ \ \ \ \ \textstyle \left(\frac{A}{e(\lambda_n(X))}-\frac{1-A}{1-e(\lambda_n(X)))}\right)(Y-\mu_A(\lambda_n(X))) + \mu_1(\lambda_n(X)) - \mu_0(\lambda_n(X))-\psi_{\lambda_n}(P) \bigg\}\\
&= \textstyle (P_n-P)\bigg\{\left(\frac{A}{\hat e(\lambda_n(X))}-\frac{1-A}{1-\hat e(\lambda_n(X))}\right)(Y-\hat \mu_A(\lambda_n(X))) + \hat \mu_1(\lambda_n(X))-\hat \mu_0(\lambda_n(X))\\
& \ \ \ \ \ \ \ \ \  \ \ \ \ \ \ \ \ \ \ \ \  \ \ \ \ \ \ \textstyle \left(\frac{A}{e(\lambda_n(X))}-\frac{1-A}{1-e(\lambda_n(X)))}\right)(Y-\mu_A(\lambda_n(X))) + \mu_1(\lambda_n(X)) - \mu_0(\lambda_n(X))\bigg\},
\end{align*}
where the last equality follows from the fact that $\psi_{\lambda_n}(\hat P_{n,\lambda_n})$ and $\psi_{\lambda_n}(P)$ are not random so $(P_n-P)\{\psi_{\lambda_n}(\hat P_{n,\lambda_n})\} = 0$ and similarly for $\psi_{\lambda_n}(P)$. 
To simplify notation, call
\begin{align*}
\hat f(Z) &= \textstyle \left(\frac{A}{\hat e(\lambda_n(X))}-\frac{1-A}{1-\hat e(\lambda_n(X))}\right)(Y-\hat \mu_A(\lambda_n(X))) + \hat \mu_1(\lambda_n(X))-\hat \mu_0(\lambda_n(X))\\
f(Z) &=\textstyle \left(\frac{A}{e(\lambda_n(X))}-\frac{1-A}{1-e(\lambda_n(X)))}\right)(Y-\mu_A(\lambda_n(X))) + \mu_1(\lambda_n(X)) - \mu_0(\lambda_n(X)).
\end{align*}
Note that
\begin{align*}
\hat f(Z) - f(Z) &= \textstyle \left(1-\frac{A}{e(\lambda_n(X))}\right)(\hat \mu_{n,1}(\lambda_n(X)) - \mu_1(\lambda_n(X))) \\
& \ \ \ \ \textstyle  + \left(1+\frac{1-A}{1-e(\lambda_n(X))}\right)(\mu_0(\lambda_n(X))-\hat \mu_{n,0}(\lambda_n(X)))\\
& \textstyle \ \ \  \ + \left(\frac{A(Y-\hat \mu_{n,1}(\lambda_n(X)))}{\hat e_n(\lambda_n(X))e(\lambda_n(X))}-\frac{(1-A)(Y-\hat \mu_{n,0}(\lambda_n(X)))}{(1-\hat e_n(\lambda_n(X))(1-e(\lambda_n(X))} \right)(e(\lambda_n(X)) - \hat e_n(\lambda_n(X))).
\end{align*}
By Assumptions 4 and 5, we have $\epsilon \leq e(\lambda_n(X)) \leq 1-\epsilon$ and $\epsilon \leq \hat e_n(\lambda_n(X)) \leq 1-\epsilon$ with $0 < \epsilon < 0.5$ and $|Y-\hat \mu_{n,a}(\lambda_n(X))| \leq C$ for some $C < \infty$ and $a \in \{0, 1\}$. Assumption 3 gives us $\|\hat \mu_{n, a}(\lambda_n(X))-\mu_{n,a}(\lambda_n(X))\| = o_P(1)$ for $a \in \{0, 1\}$ and $\|\hat e_n(\lambda_n(X)) - e(\lambda_n(X))\| = o_P(1)$. It then follows that
\begin{align*}
\|\hat f(Z) - f(Z)\| &\leq \textstyle \left(1 + \frac{1}{\epsilon}\right) \|\hat \mu_{n,1}(\lambda_n(X)) - \mu_{1}(\lambda_n(X))\| \\
& \ \ \ \ \textstyle + \left(1 + \frac{1}{1-\epsilon}\right) \|\hat \mu_{n,0}(\lambda_n(X)) - \mu_{0}(\lambda_n(X))\| \\
& \ \ \ \ \textstyle + \left(\frac{C}{\epsilon^2} + \frac{C}{(1-\epsilon)^2}\right) \|e(\lambda_n(X)) - \hat e_n(\lambda_n(X))\|\\
&= o_P(1)
\end{align*}
Next, we use the fact that since the functions $\hat e_n$, $\hat \mu_{n,a}$, and $\lambda_n$ are estimated out-of-sample (Assumption 3), then 
\begin{equation}
(P_n-P)\{\hat f(Z)-f(Z)\} =\textstyle O_P\left(\frac{\|\hat f(Z)-f(Z)\|}{\sqrt{n}}\right).
\end{equation}
To see this, we follow the general proof from \citep{kennedy2020sharp}. Namely, 
\begin{equation*}
\textstyle \E[P_n(\hat f(Z) - f(Z))] =  \E\left[\frac{1}{n}\sum_i \hat f(Z_i) - f(Z_i)\right] = \E[\hat f(Z) - f(Z)] = P\{\hat f(Z) - f(Z)\},
\end{equation*}
so $(P_n-P)\{\hat f(Z)-f(Z)\}$ has zero mean. We also have
\begin{align*}
\text{Var}[(P_n-P)\{\hat f(Z)-f(Z)\}] &= \text{Var}[P_n\{\hat f(Z)-f(Z)\}] \\
&= \textstyle \text{Var}\left[\frac{1}{n} \sum_i (\hat f(Z_i) - f(Z_i))\right]\\
&= \textstyle  \frac{1}{n}\text{Var}\left[\hat f(Z) - f(Z)\right] \\
&\leq \textstyle \|\hat f(Z) - f(Z)\|^2/n,
\end{align*} 
where the last inequality follows from the definition of variance. Then using Chebyshev's inequality,
\begin{equation}
P\left\{\frac{|(P_n-P)\{\hat f(Z)-f(Z)\}|}{\|\hat f(Z)-f(Z)\|/\sqrt{n}} \geq t \right\} \leq \frac{1}{t^2}.
\end{equation}
Since we can pick $t=1/\sqrt{\epsilon}$ for any $\epsilon > 0$, we have the result that 
\begin{equation}
(P_n-P)\{\hat f(Z)-f(Z)\} =\textstyle O_P\left(\frac{\|\hat f(Z)-f(Z)\|}{\sqrt{n}}\right) = o_P(n^{-1/2}),
\end{equation}
using the fact that $\|\hat f(Z) - f(Z)\| = o_P(1)$. We therefore have 
\begin{equation}
\hat \psi - \psi(P) = P_n\{\varphi_{\lambda^*}(Z;P)\} + T_2 + B_n + o_P(n^{-1/2}).
\end{equation}
We now analyze $T_2$. By definition we have
\begin{align*}
T_2 &=R_2(\hat P_{n,\lambda_n}, P) \\
&=  \int \left(\frac{1}{\hat e_n(\lambda_n(x))} - \frac{1}{e(\lambda_n(x))}\right)(\mu_1(\lambda_n(x)) - \hat \mu_{n,1}(\lambda_n(x)))e(\lambda_n(x))dP(x) \\
& \ \ \ \ + \int \left(\frac{1}{1-\hat e_n(\lambda_n(x))} - \frac{1}{1-e(\lambda_n(x))}\right)(\hat\mu_{n,0}(\lambda_n(x)) - \mu_0(\lambda_n(x)))(1-e(\lambda_n(x)))dP(x).
\end{align*}
Using the fact that $\epsilon \leq \hat e(\lambda_n(x)) \leq 1-\epsilon$ with probability one, we have
\begin{align*}
R_2(\hat P_{n,\lambda_n}, P) &\leq \left(\frac{1}{\epsilon}\right) \int |e(\lambda_n(x))-\hat e_n(\lambda_n(x))| * |\mu_1(\lambda_n(x))-\hat \mu_{n,1}(\lambda_n(x))| dP(x)\\
& \ \ \ \  \ \ \ + \left(\frac{1}{\epsilon}\right) \int |e(\lambda_n(x))-\hat e_n(\lambda_n(x))| * |\mu_0(\lambda_n(x))-\hat \mu_{n,0}(\lambda_n(x))| dP(x)\\
&\leq \left(\frac{1}{\epsilon}\right)\|\hat e_n(\lambda_n(X))-e(\lambda_n(X))\|\|\hat \mu_{n,0}(\lambda_n(X))-\mu_0(\lambda_n(X))\|\\
& \ \ \ \ + \left(\frac{1}{\epsilon}\right)\|\hat e_n(\lambda_n(X))-e(\lambda_n(X))\|\|\hat \mu_{n,1}(\lambda_n(X))-\mu_1(\lambda_n(X))\|.
\end{align*}
By Assumption 2 (\Cref{assumption:combined_root_n_error}), $\|\hat e_n(\lambda_n(X)) - e(\lambda_n(X))\|\|\hat \mu_{n,a}(\lambda_n(X))-\mu_{n,a}(\lambda_n(X))\| = o_P(n^{-1/2})$ for $a \in \{0, 1\}$. Therefore, $R_2(\hat P_{n,\lambda_n}, P) = o_P(n^{-1/2})$ and
\begin{equation}
\hat \psi - \psi(P) = P_n\{\varphi_{\lambda^*}(Z;P)\} + B_n + o_P(n^{-1/2}).
\end{equation}

We conclude by analyzing $B_n$. We show that $B_n = o_P(n^{-1/2})$ by the OVB assumption (\Cref{assumption:ovb}). Our proof technique follows \citep{chernozhukov2022long}. Recall:
\begin{equation}
B_n = \psi_{\lambda_n}(P) - \psi(P).
\end{equation}
Note that we can write 
\begin{align*}
\psi(P) &= \E_P[\mu_A(X) * \alpha_A(X)]\\
\psi_{\lambda_n}(P) &= \E_P[\mu_A(\lambda_n(X)) * \alpha_A(\lambda_n(X))],
\end{align*}
where recall 
\begin{equation}
    \alpha_a(x) = \textstyle \left(\frac{a}{e(x)}-\frac{1-a}{1-e(x)}\right), \qquad \alpha_a(\lambda(x)) = \textstyle \left(\frac{a}{e(\lambda(x))}-\frac{1-a}{1-e(\lambda(x))}\right).
\end{equation}
Also note that
\begin{align}
\E_P[\mu_{A}(\lambda_n(X)) * (\alpha_A(X) - \alpha_A(\lambda_n(X)))] &= 0 \label{eqn:reisz_orthogonal_1} \\
\E_P[\alpha_{A}(\lambda_n(X)) * (\mu_A(X) - \mu_A(\lambda_n(X)))] &= 0 \label{eqn:reisz_orthogonal_2}.
\end{align}
Then
\begin{align*}
B_n &= \psi_{\lambda_n}(P) - \psi(P)\\
&= \E_P\{\mu_A(\lambda_n(X))\alpha_A(\lambda_n(X)) - \mu_A(X)\alpha_A(X)\} \\
&= \E_P\{\mu_A(\lambda_n(X))\alpha_A(\lambda_n(X))\}  \\
& \ \ \ \ \ \ \ \ \ \ - \E_P\left\{[\mu_A(\lambda_n(X)) + \mu_A(X) - \mu_A(\lambda_n(X))] * [\alpha_A(\lambda_n(X))+\alpha_A(X) - \alpha_A(\lambda_n(X))]  \right\}\\
&= -\E_P\{\mu_A(\lambda_n(X))[\alpha_A(X)-\alpha_A(\lambda_n(X))]\} - \E_P\{\alpha_A(\lambda_n(X))[\mu_A(X) - \mu_A(\lambda_n(X))]\} \\
& \ \ \ \ \ \ \ \ \ \ - \E_P\{[\mu_A(X) - \mu_A(\lambda_n(X))][\alpha_A(X)-\alpha_A(\lambda_n(X))]  \}\\
&= - \E_P\{[\mu_A(X) - \mu_A(\lambda_n(X))][\alpha_A(X)-\alpha_A(\lambda_n(X))]  \}, \\
&= -\text{Cov}\left[\mu_A(X) - \mu_A(\lambda_n(X)), \alpha_A(X) - \alpha_A(\lambda_n(X))\right],
\end{align*}
where the second-to-last equality follows from \Cref{eqn:reisz_orthogonal_1,eqn:reisz_orthogonal_2} and the last equality is due to each term being zero-mean. By Assumption 1, it then follows that $B_n=o_P(n^{-1/2})$. 

Therefore, we've established
\begin{equation}
\hat \psi - \psi(P) = P_n\left\{\varphi_{\lambda^*}(Z;P)\right\} + o_P(n^{-1/2}).
\end{equation}
By the central limit theorem, it follows that
\begin{equation}
\sqrt{n}(\hat \psi - \psi(P)) \to \N\left(0, \text{Var}(\varphi_{\lambda^*}(Z;P)\right),
\end{equation}
completing the proof. 
\end{proof}
The proof for the ATT functional in Theorem 1 follows the same argument, except the influence function is given by
\begin{equation*}
    \label{eqn:att_influence_fn}
    \psi_\lambda(Z;P) = \textstyle \frac{1}{P(A=1)}\left(\left(A-\frac{(1-A)e(\lambda_n(X))}{(1-e(\lambda_n(X))}\right)(Y-\mu_0(X)) - A\int (\mu_1(X) - \mu_0(X)) dP(X|A=1)\right)
\end{equation*}
and the OVB term $B_n$ is given by 
\begin{equation}
    \label{eqn:ovb_att}
    B_n = -\text{Cov}[\mu_A(X) - \mu_A(\lambda_n(X)), \alpha_A(X) - \alpha_A(\lambda_n(X))],
\end{equation}
where
\begin{equation}
    \alpha_A(\lambda_n(X)) = \frac{1}{P(A=1)}\left(A - (1-A)\frac{e(\lambda_n(X))}{1-e(\lambda_n(X))}\right).
\end{equation}
The results from \citep{chernozhukov2018double} can be used to show that $\sqrt{n}A_n$ converges in probability to a zero-mean Gaussian random variable with variance $\text{Var}(\varphi_{\lambda^*}(Z;P))$ plus an $o_P(1)$. Meanwhile it follows from \Cref{eqn:ovb} and Assumption 1 that $\sqrt{n}B_n$ is $o_P(1)$, confirming the result. 
 
\section{Data Construction}
\label{app:data_preprocessing}

We estimate wage gaps using the Panel Study of Income Dynamics \citep{psid}, or PSID. PSID is a longitudinal survey that follows a cohort of American families frequently used to estimate gender wage gaps \citep{blau_gender_2017}. We follow \citep{blau_gender_2017} in preparing this data. While \citep{blau_gender_2017} release data for specific PSID years, our study requires a larger sample (e.g. to track workers over time). So we clean and preprocess the data, following the steps described in \citep{blau_gender_2017}. 

Specifically, we pair individuals across surveys using `intnum68' and `pernum68' as a unique label. Because different surveys use different occupation codes, we convert all occupation codes to one of the 330 `occ1990dd' codes using the crosswalk provided in \citep{david2013growth}. 
We add seven categories for when an individual's occupation is unavailable but their employment status is available. These categories are: employed, temporarily laid off, unemployed, disabled, retired, homemaker, and student. In all of our analyses, these categories are treated as special kinds of occupations. 
We also recode industry, education levels, and demographic characteristics to ensure comparability across years. 

Interviews were conducted annually from 1968 until 1997, and they have been conducted biannually since. Since the survey has been conducted biannually since 1997, we follow \citep{blau_gender_2017} and impute experience information for the skipped years using retroactive questions. Specifically, we employ a two-step procedure: we first determine if an individual worked at all, then estimate if they worked full-time. These imputations use logistic regression models with demographic and historical employment variables as predictors, estimated separately for men and women to account for gender-specific patterns in labor force participation. 

We adjust nominal wages for inflation using the Consumer Price Index, with 2015 as the base year. Ewe calculate hourly wages and censor values below \$2 per hour (in 2015 dollars) to address implausibly low wages that likely represent measurement error. We exclude self-employed individuals from wage calculations due to difficulties in accurately measuring their labor income. We define full-time work as 1500 or more annual hours. 

We note that our reported results for the raw wage gaps or regression-adjusted wage gaps are not identical to those reported in \citep{blau_gender_2017}. This is for a few reasons: first, we consider a cohort of workers between 1990-2019, while \citep{blau_gender_2017} mainly focus on four individual-year cohorts. Second, our analysis in the main sample considers a trimmed population in order to encourage overlap \citep{crump2006moving}. Additionally, we use cross-fitting for all our analyses to prevent overfitting \citep{chernozhukov2018double}. Moreover, \citep{blau_gender_2017} adjust for how urban an individual's area is, which we do not because the relevant variable requires access to the PSID restricted sample. However, we note that \citep{blau_gender_2017} do not find that location variables explain much of the wage gap, e.g. explaining 0.0\% of the gap in 1980 and 0.3\% in 2010.

\begin{table}[p]
  \footnotesize
  \centering
  \caption{The PSID sample used for wage gap analysis differs slightly from the one used in \citep{blau_gender_2017}.
  }
  \begin{tabular}{ l  c c c c}
   & \multicolumn{2}{c}{\textbf{Our sample}} & \multicolumn{2}{c}{\textbf{Blau and Kahn (2017) sample}} \\
  \toprule
       & Unadjusted ratio & Adjusted ratio & Unadjusted ratio & Adjusted ratio \\
       \midrule
  1989 & 0.740 & 0.904 & 0.740 & 0.924 \\
  1998 & 0.768 & 0.915 & 0.772 & 0.914 \\
  2010 & 0.797 & 0.921 & 0.793 & 0.916 \\
  \bottomrule
 \end{tabular}
  \label{tab:app:blau_kahn_ratio_comparison}
\end{table}

Additionally, while we attempt to follow \citep{blau_gender_2017} in constructing our sample, we ended up with a slightly different sample due to ambiguities in variable definitions and how to impute missing experience. \Cref{tab:app:blau_kahn_ratio_comparison} shows a comparison of wage gaps imputed from our sample and from the sample in \citep{blau_gender_2017}. To make results comparable, we do not perform trimming or use cross-fitting for either. The adjusted wage ratio is computed using the linear regression described in \Cref{sec:predictivePSID} with coarse-grained occupational encodings. We adjust for the full set of variables available in each sample, which are identical except we include the MSA indicator in \citep{blau_gender_2017} for denoting how rural a population is. The unadjusted ratios are all very similar, differing by $0.004$ at most. There is more variation in the adjusted ratios, which differ by as much as $0.020$ in 1989 but by only $0.001$ and $0.005$ in 1998 and 2010, respectively.

\section{Model and Training Details}
\label{sec:app:model_and_training_details}
Here, we describe the fine-tuning approaches in more detail. We use the CAREER model as an initial representation \citep{vafa2023career}. We use $D=64$ dimensions for the representation, 4 encoder layers, 4 attention heads, and 256 hidden units for the feedforward neural networks. 
For our final model, we pretrain 16 models and perform fine-tuning over all pretrained seeds to form a final ensembled model \citep{dietterich2000ensemble,lakshminarayanan2017simple}. 
This is considerably smaller than the transformers used to train large language models, which have billions of parameters  \citep{vaswani2017attention,Devlin2019Bert,ouyang2022training}. We find that smaller transformer models are effective at modeling sequence of occupations, which are generally shorter and less complex than the sequences of words large language models are applied to. 

We first fit the model to a dataset of 24 million resumes (``pretraining''). Our pretraining objective follows \citep{vafa2023career}: predict the next job each individual will have. 

We consider four different fine-tuning approaches to adjust the representation on survey data used for wage gap estimation. When we fine-tune, we model wage as a function of the representation $\lambda$ and covariates $Z \in \R^P$:
\begin{equation*}
    \hat \mu_A(\lambda(X), Z) = \rho_G(\lambda(X)) + \beta_G^\top Z,
\end{equation*}
where $\rho_G$ is a two-layer feedforward neural network with 64 hidden units and $\beta_G \in \R^P$ is a vector of regression coefficients. 

In supervised fine-tuning, our objective is to minimize the predictive error of wage: 
\begin{equation}
   \label{app:eqn:supervised_fine_tuning}
    \frac{1}{N}\sum_{i=1}^N (Y_i - \hat \mu_{A_i}(\lambda(X_i), Z_i))^2.
\end{equation}
We optimize this objective with respect to the representation $\lambda$, feedforward neural network paramaters comprising $\rho$, and regression coefficients $\beta$. We optimize using Adam \citep{kingma2014adam} and use early stopping based on a validation set. 

In multitask fine-tuning \citep{shi2019adapting,chernozhukov2022riesznet}, we introduce a propensity model as a function of the representation and covariates:
\begin{equation*}
    \label{app:eqn:propensity_model}
    \hat e(\lambda(X), Z) = \sigma(\gamma^\top \lambda(X))
\end{equation*}
where $\sigma(\cdot)$ is the inverse-logit function and $\gamma \in \R^P$ are regression coefficients. 
Our objective is to minimize the combined predictive error of wage and group label:
\begin{equation*}
    \frac{1}{N}\sum_{i=1}^N (Y_i - \hat \mu_{A_i}(\lambda(X_i), Z_i))^2 - \eta * [A_i \log \hat e(\lambda(X_i), Z_i) + (1-A_i) \log(1- \hat e(\lambda(X_i), Z_i))],
\end{equation*}
where $\eta \in \R^+$ is a hyperparameter that controls the tradeoff between wage and propensity losses. Following \citep{shi2019adapting}, we set $\eta$ to 1. 
We optimize this objective with respect to the representation $\lambda$, the parameters comprising $\rho$, and the regression coefficients $\gamma$.  We optimize using Adam \citep{kingma2014adam} and use early stopping based on a validation set.

In projection fine-tuning, we alternate between minimizing \Cref{app:eqn:supervised_fine_tuning} until the loss converges on a held-out set and minimizing the propensity loss until its loss converges on a held-out set:
\begin{equation}
    \label{app:eqn:propensity_loss}
    -\frac{1}{N}\sum_{i=1}^N [A_i \log \hat e(\lambda(X_i), Z_i) + (1-A_i) \log(1- \hat e(\lambda(X_i), Z_i))].
\end{equation}
We optimize each objective with respect to all trainable parameters using Adam and perform early stopping with a validation set. 
In contrast to multi-task fine-tuning, this procedure doesn't require setting a parameter to dictate the tradeoff between propensity and wage losses. In practice, we find that the loss worsens after the second round of wage updates, so we end the procedure there.

For difference-based fine-tuning, we first fit a conditional wage function $\hat m(\lambda_m(x), z)$ to model the conditional wage averaged over the two groups $\E[Y|X=x, Z=z]$. We use
\begin{equation*}
    \hat m(\lambda_m(X), Z) = \rho(\lambda_m(X)) + \beta^\top Z,
\end{equation*}
where $\rho$ is a 2-layer neural network and $\beta \in \R^P$ a vector of regression coefficients. We then fit $\hat m$ and $\lambda_m(X)$ by minimizing the following objective
\begin{equation*}
    \frac{1}{N}\sum_{i=1}^N (\hat m(\lambda_m(X_i), Z_i) - Y_i)^2.
\end{equation*}
We then fit a propensity model $\hat e(\lambda_e(X), Z)$ by minimizing \Cref{app:eqn:propensity_loss}. 
We then introduce a new model $\hat \tau$ to capture the difference in group means:
\begin{equation*}
    \hat \tau(\lambda(X), Z) = \rho(\lambda(X)) + \beta^\top Z,
\end{equation*}
where $\rho$ is a two-layer feedforward neural network and $\beta \in \R^P$ is a vector of regression coefficients. We then minimize the following objective:
\begin{equation*}
    \frac{1}{N}\sum_{i=1}^N [(Y_i - \hat m(\lambda_m(X_i), Z_i)) - (A_i - \hat e(\lambda_e(X_i), Z_i))\tau(\lambda(X_i), Z_i)]^2,
\end{equation*}
with respect to the parameters in $\tau(\lambda(X), Z)$, \textbf{keeping fixed} the parameters underlying $\hat m(\lambda_m(X), Z)$ and $\hat e(\lambda_e(X), Z)$. We again optimize this objective using Adam and use a held-out validation set to monitor for early stopping. 

\Cref{tab:predictive_accuracy_full} shows predictive accuracy metrics. All numbers are with 5-fold cross-fitting.  Because gender is encoded as a binary variable, we use pseudo $R^2$ for gender, $1-\frac{NLL_{\text{null}}-NLL_m}{NLL_{\text{null}}}$, where $NLL_{\text{null}}$ is the negative log-likelihood of a null model that assigns 0.5 probability to each individual, and $NLL_m$ is the negative log-likelihood of the given model (all on held-out data with 5-fold cross splitting). 

\Cref{tab:wage_gaps} shows each model's estimate for the unexplained wage ratio, given by exponentiating each model's estimate of the unexplained log-wage gap (following \citep{blau_gender_2017}). Test-set bootstrapping is done on the final metric; for example. we bootstrap a dataset, compute the wage ratio for that dataset, and include that in our standard deviation. The R-Learner metric is each model's performance on the R-learner metric given in \Cref{eqn:r_learner_eval_metric}, all evaluated on held-out data. The last column is the difference between each method's R-learner metric and the R-learner metric from difference-based fine-tuning.

\section{Semi-Synthetic Details}
\label{app:semi_synthetic_details}
We first obtain a representation $\lambda^*: \mathcal X \to \R^m$ by overfitting a transformer 10x the size of CAREER using supervised fine-tuning (\Cref{app:eqn:supervised_fine_tuning}). 
For each experiment, we're given a true gap $\tau \in \R$, a noise scale $\sigma^2$, and a shared proportion $0 \leq \pi \leq 1$. 
We first randomly select an index set $\mathcal I_A$ of $m/2$ indices between 1 and $m$ uniformly at random to use for the propensity model. 
We then select an index set $\mathcal I_Y$ for the wage model: $\pi * m/2$ are sampled uniformly at random from $\mathcal I_A$, and then $(1-\pi)*m/2$ are sampled uniformly at random from the indices not in $\mathcal I_A$. Denote by $\mathbf{u} \in \{0, 1\}^m$ the vector such that $u_i=1$ if and only if $i \in \mathcal I_A$. Similarly, denote by $\mathbf{v} \in \{0, 1\}^m$ the vector such that $v_i = 1$ if and only if $i \in \mathcal I_Y$. We then simulate a vector of regression coefficients $\beta$ such that $\beta_i \sim \N(0, (1/5)^2)$. We then sample group membership labels via
\begin{equation*}
    A_i \sim \text{Bern}\left(\sigma\left(\textstyle \sum_j u_j\beta_j* \lambda(X_i)_j\right)\right).
\end{equation*}
We then sample noise terms $\epsilon_i \sim \N(0, \sigma^2)$. Finally, to get wages, we set
\begin{equation*}
     Y_i = \textstyle \tau A_i + \sum_j v_j \beta_j  * \lambda(X_i)_j + \epsilon_i.
\end{equation*}

\begin{figure}[p]
\centering
\includegraphics[width=0.9\linewidth]{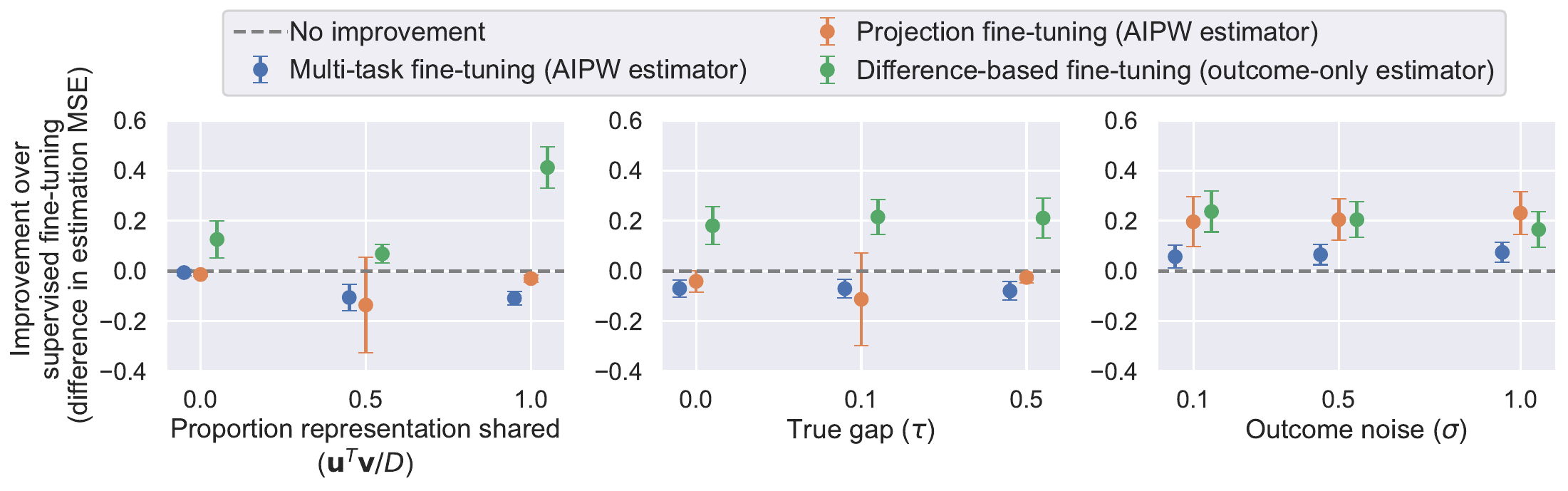}
\caption{The AIPW estimator on semi-synthetic data performs worse than the outcome-only estimator. For each semi-synthetic experiment, the true unexplained wage gap is known, and each method provides a different estimate of this gap. This figure compares each method's average error for estimating this gap, evaluated via MSE between the true and estimated unexplained gap. Specifically, the Y-axis compares each method's estimation error to the error from estimates derived from a model using the outcome-only model and standard supervised fine-tuning (larger values on the Y-axis correspond to larger improvements). Bars represent 95\% confidence intervals.
}
\label{app:fig:semi_synthetic_experiments_aipw}
\end{figure}

\begin{figure}[p]
\centering
\includegraphics[width=0.9\linewidth]{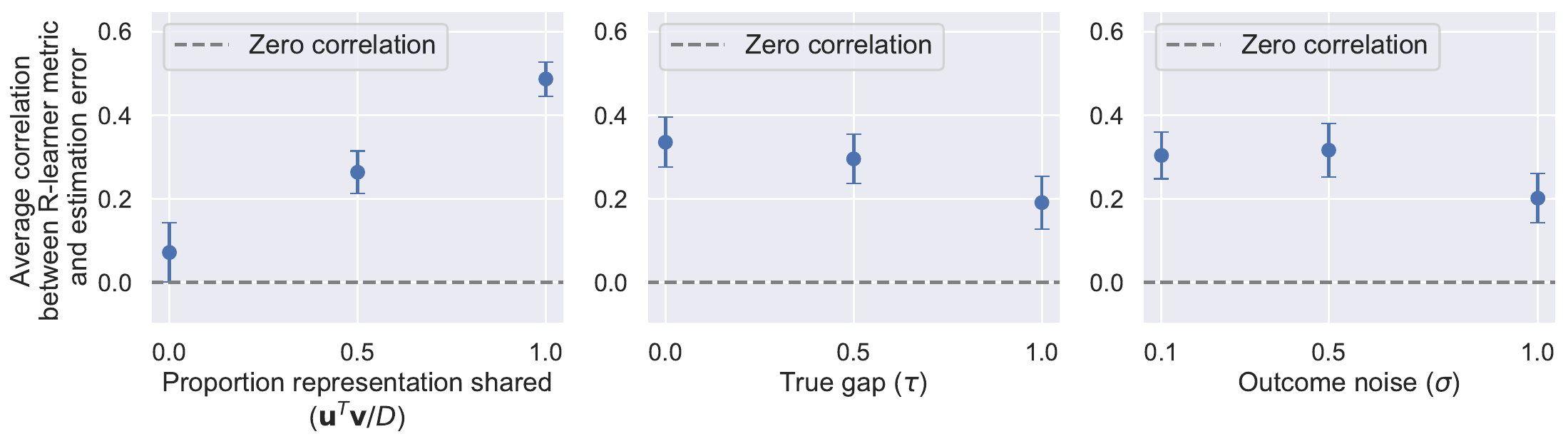}
\caption{
The R-learner evaluation metric is correlated with model estimation error across semi-synthetic experiments. The highest correlation occurs when more of the representation is shared in the underlying data. The models used to calculate this correlation are: linear regression, supervised fine-tuning, and the three debiased fine-tuning approaches described in \Cref{sec:framework}. Test-set bootstrapped standard errors are in parentheses. The full results for each of the 27 settings is in \Cref{tab:app:all_semi_synthetic_results}.
}
\label{app:fig:semi_synthetic_evaluation_metrics}
\end{figure}

\begin{figure}[p]
\centering
\includegraphics[width=0.9\linewidth]{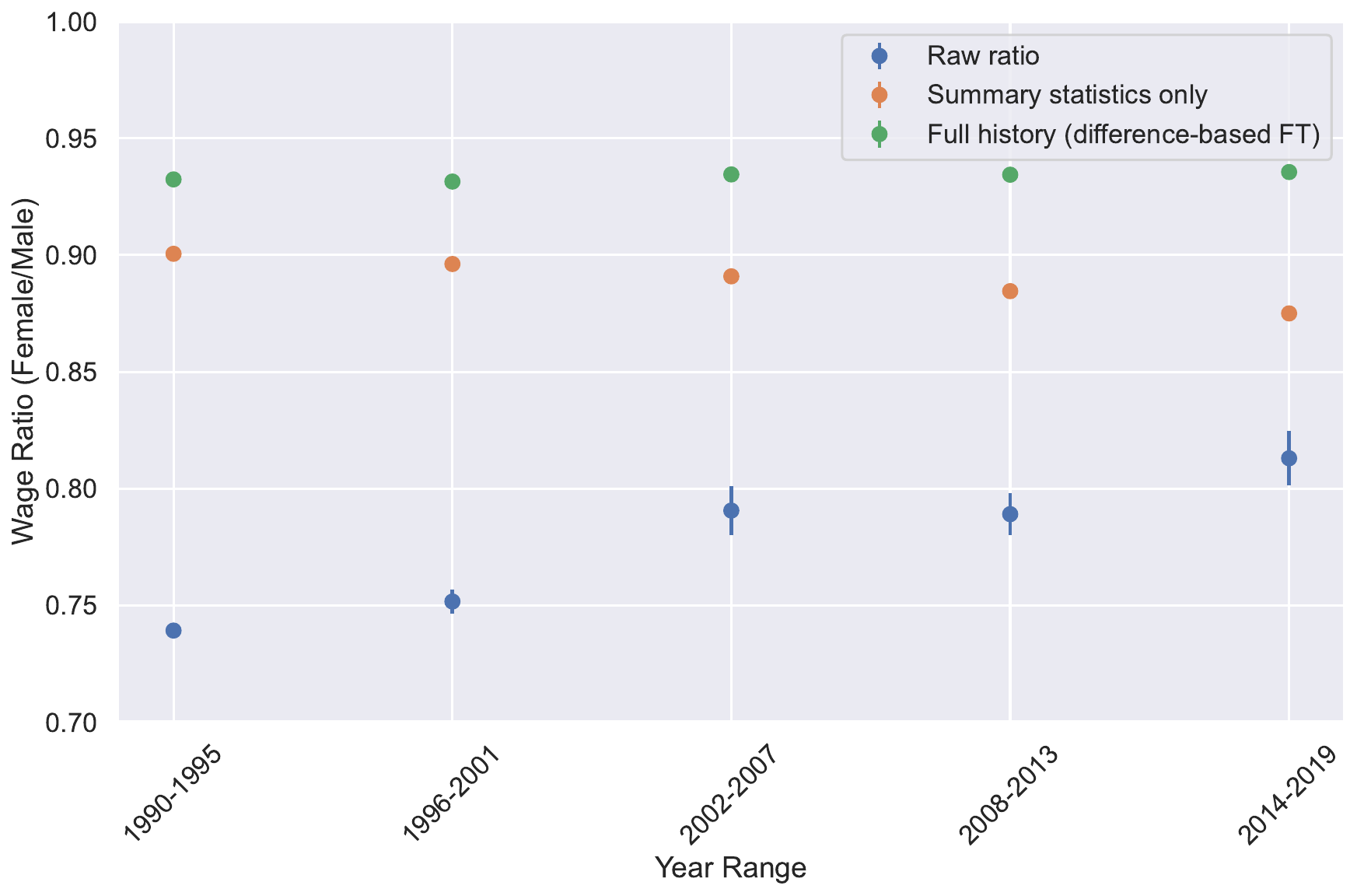}
\caption{Compared to methods that adjust the gender wage gap for summary statistics of history, the learned representations of history explain more of the gap for later years in the PSID survey. Test-set bootstrapped standard errors are in parentheses.
}
\label{app:fig:wage_ratios_over_time}
\end{figure}

\begin{figure}[p]
\centering
\includegraphics[width=0.9\linewidth]{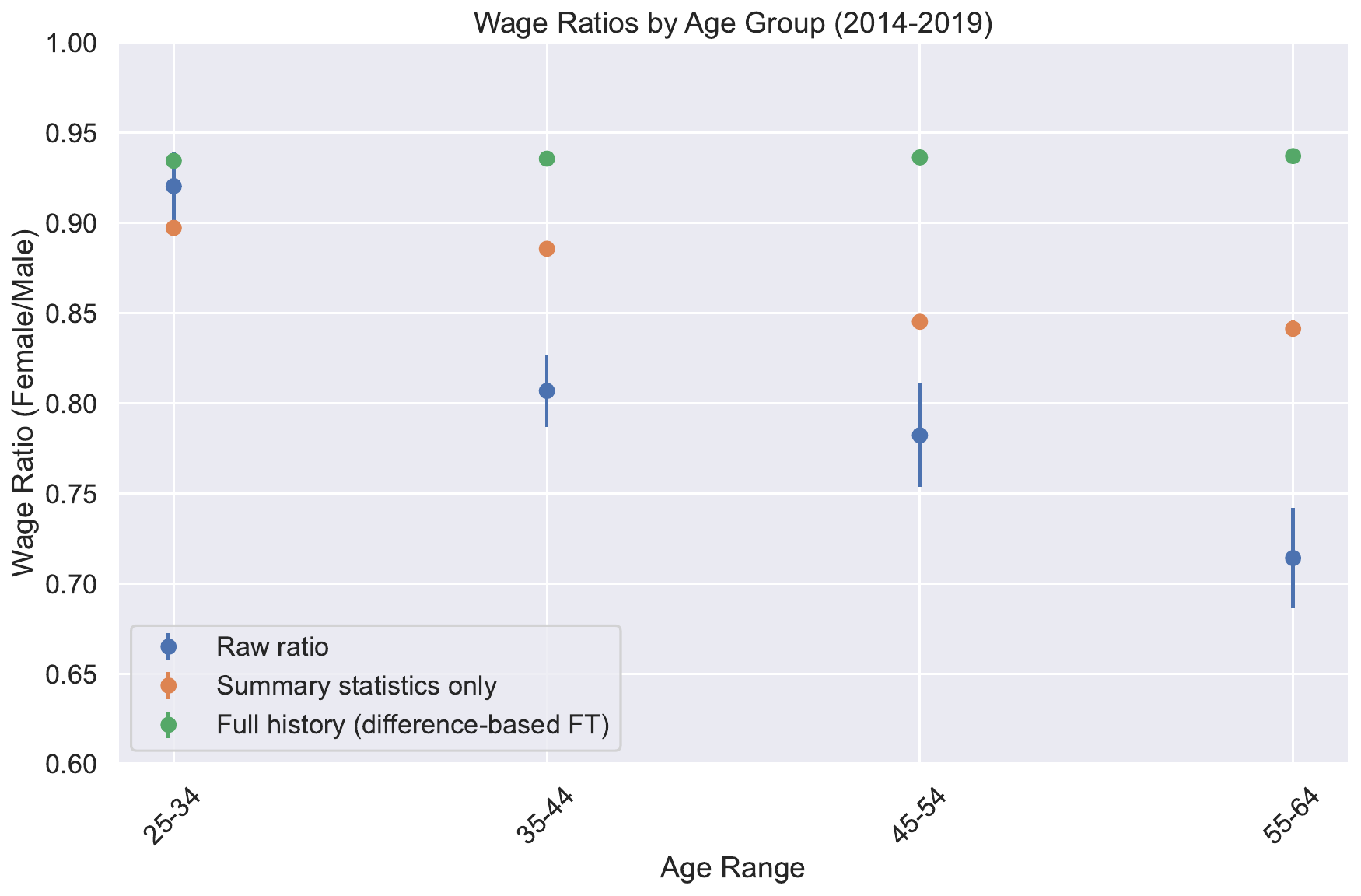}
\caption{For a fixed time period (2014-2019), history explains more of the gender wage gap for older workers than for younger workers. Test-set bootstrapped standard errors are in parentheses.
}
\label{app:fig:wage_ratios_by_age}
\end{figure}

\begin{figure}[p]
\centering
\includegraphics[width=0.9\linewidth]{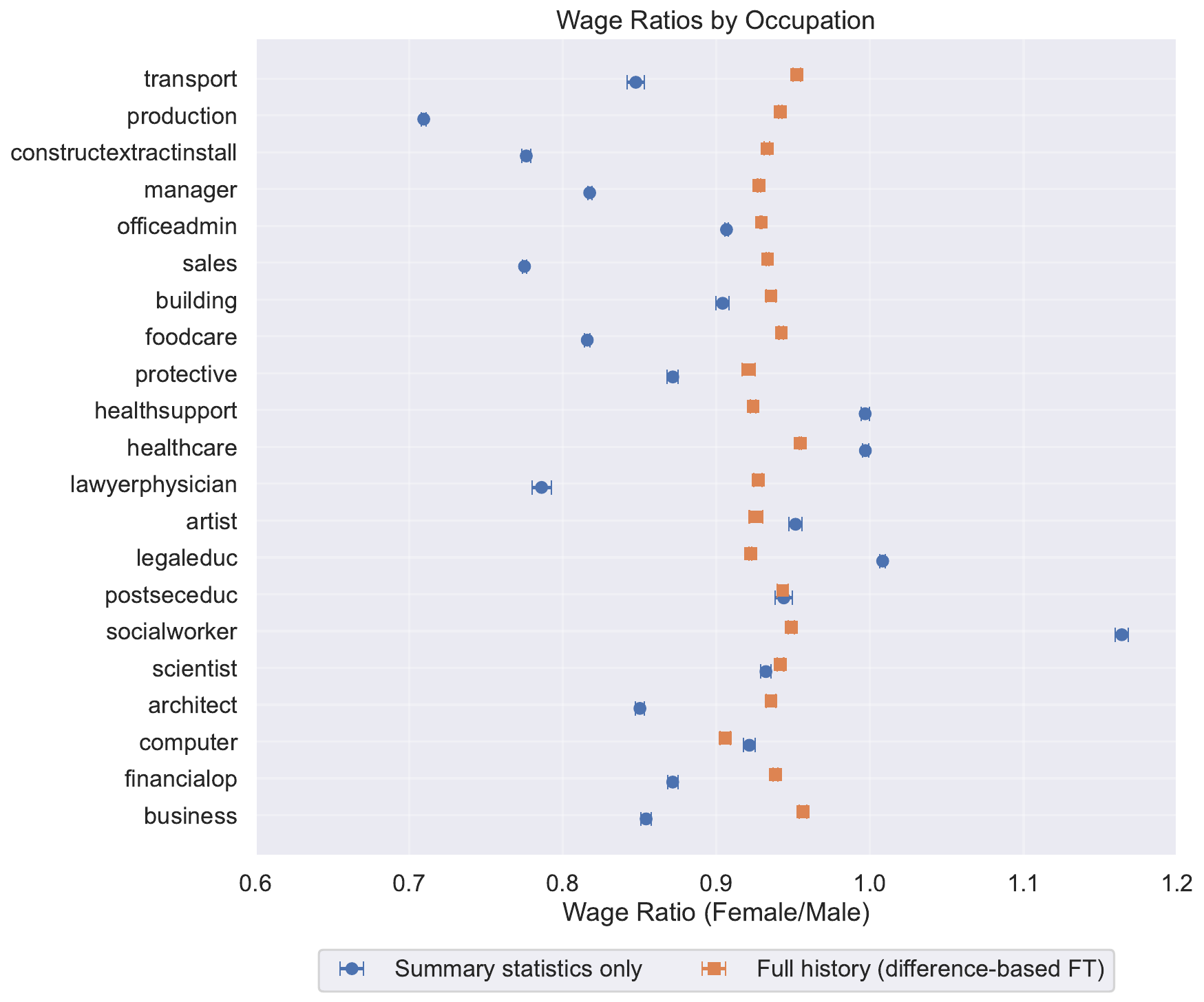}
\caption{Compared to the gender wage gap explained by summary statistics of history, the wage gap explained by representations of full history is smoother across the occupational categories defined by \citep{blau_gender_2017}. Bars represent test-set bootstrapped standard errors.
}
\label{app:fig:wage_ratios_by_occupation}
\end{figure}

We consider three different values for the effect parameter $\tau$ ($0.0$, $0.1$, and $1.0$), three different values for the shared proportion parameter $\pi$ ($0.0$, $0.5$, and $1.0$), and three different values for the noise scale $\sigma^2$ ($0.1$, $0.5$, and $1.0$). For each of the 3x3x3=27 combinations of settings, we simulate 10 datasets, resulting in 270 datasets for the semi-synthetic study. For each of the 270 datasets, we train models using each of the approaches described in the main text. Each point in \Cref{fig:semi_synthetic_experiments} shows the error using the outcome-only estimator averaged over datasets with each setting; for example, there are 90 datasets where the proportion of shared representations is 1.0, and the dashed line represents the averaged supervised fine-tuning estimation error for these datasets while the bars represent 95\% confidence intervals for these datasets. Full experimental results for all 27 settings are in \Cref{tab:app:all_semi_synthetic_results}.

The semi-synthetic results in \Cref{sec:empirical_validation} use the outcome-only estimator. \Cref{app:fig:semi_synthetic_experiments_aipw} compares performance across settings using the AIPW estimator. We do not compute AIPW estimates for representations from the difference-based fine-tuning approach because they're optimized for outcome-only estimators (\Cref{eqn:r_learner_outcome_model}). All estimators perform worse than the outcome-only analogues. 

To validate the R-Learner metric, we compute its correlation with estimation error for the semi-synthetic data. For each of the 270 datasets in our semi-synthetic study, we compute the correlation between the R-learner metric for the five estimation methods (coarse-grained regression, supervised fine-tuning, multitask fine-tuning, projection fine-tuning, and difference-based fine-tuning) and the estimation MSE for each method (i.e. a correlation between two length-5 vectors). This gives us 270 different correlations. \Cref{app:fig:semi_synthetic_evaluation_metrics} plots the correlation averaged datasets; for example, there are 90 datasets where the proportion of shared representations is 1.0, and so we plot the correlation averaged over those 90 datasets. The full results for each of the 27 settings is in \Cref{tab:app:all_semi_synthetic_results}. These results show that performance on the R-Learner metric is correlated with estimation error. 

\section{Details for Qualitative Examples}
\label{app:sec:qualitative_details}
The qualitative exercise in \Cref{sec:gapanalysis} used a clustering method to identify clusters that were important for the wage gap. Here we provide more details. 

We first form clusters of histories by dividing an individual's current occupation into partitions based on the representations of history. For this exercise, we use the representations from the difference-based fine-tuning model. Each occupation belongs to one of 21 coarse-grained categories. We aggregate all observations for each occupational category, and then partition the histories for these observations into 15 clusters, first using t-SNE \citep{van2008visualizing} to project each representation into 2-dimensional space, and then using K-Means to form 15 clusters for each coarse-grained occupation.

Our goal is to find the clusters of history that best enhance the predictive power of the a LASSO model. Instead of refitting the LASSO model, we aim to predict the difference between the original model's wage prediction and the actual wage, using these clusters. If a cluster can accurately predict this difference on new, unseen data, it means that it would improve the original model's predictions.

We use a regression tree to add history clusters to the predictions of the LASSO model. A regression tree incorporates one cluster at a time into its prediction of wages. At each step, it selects the cluster that would most decrease the training MSE if added.  As a result, the sequence in which the regression tree selects clusters provides us with an understanding of their relative importance.

\begin{table}[p]
  \footnotesize
  \centering
  \caption{The top 15 omitted variables as identified by a regression tree as being most predictive of wage.  These variables are correlated with both wage and gender, meaning they result in omitted variable bias when excluded from traditional approaches. The $R^2$ column reports the held-out wage $R^2$ after each feature has been added (so e.g. the $R^2$ value for the 5th added variable is for a model that includes the first 5 variables). 
  }
  \resizebox{\textwidth}{!}{
  \begin{tabular}{c c l l c c}
    \toprule
  Order added & Coarse-grained occupation & Most prevalent current job & Most prevalent previous jobs & Frac. female & $R^2$ \\
  \midrule
   \textbf{1}& manager & manager/administrator & \makecell[l]{manager/administrator\\ salesperson\\ marketing manager} & 0.32 & 0.4558 \\ \midrule
   \textbf{2}& manager & manager/administrator & \makecell[l]{secretary/stenographer\\ bookkeeper\\ admin. support job} & 0.79 & 0.4563 \\ \midrule
   \textbf{3}& education/legal/library & primary school teacher & \makecell[l]{child care worker\\ primary school worker\\ teacher's aide} & 0.80 & 0.4565 \\ \midrule
   \textbf{4}& sales & retail salesperson & \makecell[l]{retail salesperson\\ truck driver\\ marketing manager} & 0.31 & 0.4565 \\ \midrule
   \textbf{5}& manager & manager/administrator & \makecell[l]{computer scientist\\ engineering technician\\ production supervisor} & 0.18 & 0.4568 \\ \midrule
   \textbf{6}& manager & manager/administrator & \makecell[l]{human resource clerk\\ bookkeeper\\ admin. support job} & 0.59 & 0.4569 \\ \midrule
   \textbf{7}& manager & purchasing manager & \makecell[l]{purchasing manager\\ marketing manager\\ salesperson} & 0.32 & 0.4572 \\ \midrule
   \textbf{8}& office/admin. support & mail carrier & \makecell[l]{mail carrier\\ postal clerk\\ mail and paper handler} & 0.50 & 0.4573 \\ \midrule
   \textbf{9}& nurse/healthcare practitioner & optical good worker & \makecell[l]{clinical lab technician\\ health technician\\ health/nursing aide} & 0.79 & 0.4574 \\ \midrule
   \textbf{10}& production occupations & lathe operative & \makecell[l]{precision/craft worker\\ textile machine operator\\ lathe operative job} & 0.24 & 0.4573 \\ \midrule
   \textbf{11}& transportation/materials moving & truck driver & \makecell[l]{shipping/receiving clerk\\ freight laborer\\ machine operator} & 0.07 & 0.4573 \\ \midrule 
   \textbf{12}& manager & manager/administrator & \makecell[l]{industrial engineer\\ production supervisor\\ engineering technician} & 0.10 & 0.4574 \\ \midrule
   \textbf{13}& manager & financial manager & \makecell[l]{financial manager\\ accountant/auditor\\ bookkeeper} & 0.49 & 0.4573 \\ \midrule
   \textbf{14}& transportation/materials moving & hand packer/packager & \makecell[l]{hand packer/packager\\ garage/service-stated occupation\\ packer/filler} & 0.51 & 0.4573 \\ \midrule
   \textbf{15}& construction/extraction/installation & installation worker & \makecell[l]{mason/tiler\\ roofer/slater\\ operating engineer} & 0.10 & 0.4573 \\ 
   \bottomrule
   \end{tabular}
   }
   \label{tab:app:all_15_top_clusters}
\end{table}

\begin{table}[p]
  \footnotesize
  \centering
  \caption{The three most predictive clusters of wage in each of the three most common fine-grained occupations provided by the occ1990dd taxonomy \citep{david2013growth}, considering only those with at least an 80/20 gender balance. These variables are correlated with both wage and gender, meaning they result in omitted variable bias when excluded from traditional approaches.
  }
  \begin{tabular}{c l c}
    \toprule
  Fine-grained occupation & Most prevalent previous jobs & Frac. female \\
  \midrule
    &  \makecell[l]{economist\\ marketing manager\\ salesperson} & 0.21 \\ \cmidrule(l){2-3}
   \textbf{manager/administrator} & \makecell[l]{mechanical engineer\\ engineer\\ electrical engineer} & 0.13 \\ \cmidrule(l){2-3}
   & \makecell[l]{secretary\\ bookkeeper\\ admin. support job} & 0.91   \\ 
   \midrule
  &  \makecell[l]{freight laborer\\ cashier\\ retail salesperson} & 0.59 \\ \cmidrule(l){2-3}
   \textbf{retail salesperson} & \makecell[l]{child care worker\\ secretary\\ customer service rep} & 0.70 \\ \cmidrule(l){2-3}
   & \makecell[l]{machine operator\\ truck driver\\ freight laborer} & 0.29 \\ 
   \midrule
           &  \makecell[l]{child care worker\\ health and nursing aide\\ cashier} & 0.80 \\ \cmidrule(l){2-3}
   \textbf{machine operator} & \makecell[l]{textile machine operator\\ material recording clerk\\ punching/stamping press operator} & 0.22 \\ \cmidrule(l){2-3}
   & \makecell[l]{lathe operator\\ slicing/cutting/crushing machine operator\\ pucnhing/stamping press operator} & 0.39  \\ 
   \bottomrule
   \end{tabular}
   \label{tab:app:fine_grained_omitted_variables}
\end{table}

What are the characteristics of the most important clusters? Rather than exhaustively enumerate each occupation in a cluster, we use a heuristic to identify the most prevalent current jobs and historical jobs for each cluster. Specifically, for each cluster and job, we compute the proportion of histories in the cluster that contain the job. We also compute the proportion of histories in the broader occupational category that contain the job. We take the jobs that are most prevalent in the cluster relevant to the rest of the broad occupational category, only taking jobs that are present in more than 2.5\% of clusters in the history. We compute an analogous heuristic for current jobs. We only include occupations that are encoded as occ1990dd categories, including the seven additional broad categories PSID includes in its taxonomy (e.g. ``employed''). 

\Cref{fig:history_clusters} shows the top clusters identified by the regression tree. Most of these clusters are formed by partitioning manager jobs into finer-grained histories. These clusters reveal important aspects of history that refine an individual's current occupation. For example, one cluster consists of managers who were previously computer scientists and engineering technicians. Managers with these jobs in their history get paid more than managers without them. These refinements may reflect differences in occupations that are not captured by the initial encoding (e.g. there is no occupational category for ``engineering manager''). While models that do not incorporate history omit these factors, they are captured by our model. 
\Cref{tab:app:all_15_top_clusters} shows all top 15 clusters along with the held-out log-wage MSE when adding history clusters to the original predictions of the LASSO model. Even when all clusters are included in the regression tree, the performance does not reach that of the full CAREER model. There are a few reasons for this. CAREER, which learns continuous representations for each history, is more flexible than models that treat groups of similar histories discretely. All histories in a cluster are not identical. So by differentiating between histories in a cluster, CAREER can more flexibly model wages than a model that treats all histories in a cluster the same. Additionally, although each cluster is treated separately by the regression tree, CAREER's continuous representations allow the model to pool information across histories. For example, if software engineering managers and hardware engineering managers are clustered into different categories, the regression tree cannot use the wages of software engineering managers to help predict those of hardware engineering managers. Meanwhile, CAREER is able to pool this information since its representations of history are not discrete. \Cref{tab:app:fine_grained_omitted_variables} extends the analysis by recomputing clusters for each of the three most-common fine-grained occupations (only considering those with at least an 80/20 gender balance). Again, these clusters are predictive of both wage and gender, resulting in omitted variable bias when they are not included in the analysis. 

\section{Additional Tables}
\label{app:sec:additional_tables}

\begin{table}[p]
  \footnotesize
  \centering
  \caption{The estimated gender wage ratio varies somewhat based on the threshold used to clip individual with extreme propensity scores. Across all thresholds, adjusting for full history explains more of the gap. Data is for the PSID sample from 1990-2019. Test-set bootstrapped standard errors are in parentheses.
  }  
  \resizebox{\textwidth}{!}{
    \begin{tabular}{llcccc}
    & & \multicolumn{4}{c}{\textbf{Clipping threshold}} \\
    \cmidrule{3-6}
      & \textbf{Adjustment method}  & 0.00 & 0.01 & 0.025 & 0.05 \\
    \midrule
    Data size & --- & 91,568 & 77,098 & 64,282 & 53,618 \\
    \midrule
    Unadjusted & --- & 0.771 (0.002) & 0.775 (0.004) & 0.781 (0.005) & 0.792 (0.007) \\
    \midrule
    Adjusted for summary statistics & Linear regression & 0.891 (0.001) & 0.886 (0.001) & 0.884 (0.001) & 0.880 (0.001) \\
    \midrule
    Adjusted for full history & Multi-task fine-tuning & 0.907 (0.001) & 0.907 (0.001) & 0.909 (0.001) & 0.910 (0.001) \\
    & Projection fine-tuning & 0.905 (0.001) & 0.904 (0.000) & 0.905 (0.001) & 0.906 (0.001) \\
    & Difference-based fine-tuning & 0.933 (0.000) & 0.934 (0.000) & 0.934 (0.000) & 0.934 (0.000) \\
    \bottomrule
    \end{tabular}}
  \label{tab:app:clipping_comparison}
\end{table}

\begin{table}[p]
  \footnotesize
  \centering
  \caption{The gender wage ratio depends on whether it is estimated with the outcome-only estimator or the AIPW estimator. For both estimators, adjusting for full history explains more of the gap. The unadjusted wage ratio is 0.775.  Data is for the PSID sample from 1990-2019. Test-set bootstrapped standard errors are in parentheses.
  }
  \begin{tabular}{ll  c c}
  \toprule
  Adjustment method & & Outcome-only & AIPW \\
       \midrule
  Adjusted for summary statistics & Linear regression & 0.886 (0.001) & 0.868 (0.017) \\ \midrule
  Adjusted for full history & Multitask fine-tuning & 0.907 (0.001) & 0.902 (0.012) \\
  & Projection fine-tuning & 0.904 (0.000) & 0.900 (0.008) \\
  \bottomrule
 \end{tabular}
  \label{tab:app:outcome_only_vs_AIPW}
\end{table}

The wage gap estimates in \Cref{sec:empirical_validation} use the outcome-only estimator. \Cref{tab:app:outcome_only_vs_AIPW} compares wage gap estimates from the outcome-only estimator to those made using the AIPW estimator (\Cref{eqn:aipw}). All results use 5 folds for cross-fitting and $0.01$ clipping. We do not compute AIPW estimates for representations from the difference-based fine-tuning approach because they're optimized for outcome-only estimators (\Cref{eqn:r_learner_outcome_model}).

\begin{table}[p]
  \footnotesize
  \centering
  \caption{Estimation MSE for semi-synthetic experiments. Each of the 27 settings is averaged over 10 random seeds.
  }
\resizebox{\textwidth}{!}{
\begin{tabular}{ccccccccc}
\multicolumn{3}{c}{\textbf{Settings}} & \multicolumn{4}{c}{\textbf{Estimation Error}} & \textbf{Validation metric}\\
\toprule
Treatment Effect & Shared Proportion & Noise Scale & Regression & Multi-task & Projection & Difference-based & Avg. R-Learner/Error Corr.\\
\midrule
0.0 & 0.0 & 0.1 & -0.201 (0.463) & -0.081 (0.214) & 0.010 (0.392) & 0.119 (0.381) & 0.150 (0.826) \\
0.0 & 0.0 & 0.5 & -0.205 (0.485) & 0.017 (0.052) & 0.011 (0.296) & 0.084 (0.480) & 0.132 (0.709) \\
0.0 & 0.0 & 1.0 & -0.198 (0.440) & 0.022 (0.072) & 0.089 (0.373) & 0.083 (0.420) & 0.182 (0.707) \\
0.0 & 0.5 & 0.1 & -0.167 (0.318) & 0.072 (0.160) & 0.072 (0.082) & 0.069 (0.168) & 0.284 (0.575) \\
0.0 & 0.5 & 0.5 & -0.182 (0.340) & 0.061 (0.176) & 0.075 (0.114) & 0.047 (0.102) & 0.417 (0.458) \\
0.0 & 0.5 & 1.0 & -0.205 (0.339) & 0.038 (0.124) & 0.051 (0.160) & 0.043 (0.090) & 0.333 (0.601) \\
0.0 & 1.0 & 0.1 & -0.382 (0.372) & 0.147 (0.237) & 0.470 (0.493) & 0.490 (0.476) & 0.580 (0.307) \\
0.0 & 1.0 & 0.5 & -0.455 (0.387) & 0.064 (0.200) & 0.438 (0.363) & 0.332 (0.289) & 0.548 (0.411) \\
0.0 & 1.0 & 1.0 & -0.401 (0.431) & 0.133 (0.175) & 0.498 (0.429) & 0.353 (0.340) & 0.395 (0.510) \\
0.1 & 0.0 & 0.1 & -0.224 (0.385) & -0.073 (0.180) & -0.023 (0.150) & 0.104 (0.242) & 0.144 (0.718) \\
0.1 & 0.0 & 0.5 & -0.182 (0.344) & -0.028 (0.119) & 0.099 (0.326) & 0.150 (0.293) & 0.051 (0.756) \\
0.1 & 0.0 & 1.0 & -0.148 (0.314) & 0.013 (0.080) & 0.171 (0.395) & 0.178 (0.327) & 0.186 (0.671) \\
0.1 & 0.5 & 0.1 & -0.188 (0.363) & 0.071 (0.224) & 0.069 (0.106) & 0.086 (0.146) & 0.278 (0.516) \\
0.1 & 0.5 & 0.5 & -0.186 (0.385) & 0.029 (0.043) & 0.072 (0.071) & 0.052 (0.171) & 0.303 (0.569) \\
0.1 & 0.5 & 1.0 & -0.187 (0.328) & 0.081 (0.235) & 0.058 (0.109) & 0.066 (0.109) & 0.209 (0.442) \\
0.1 & 1.0 & 0.1 & -0.373 (0.382) & 0.196 (0.198) & 0.581 (0.567) & 0.449 (0.530) & 0.501 (0.372) \\
0.1 & 1.0 & 0.5 & -0.457 (0.398) & 0.077 (0.171) & 0.453 (0.377) & 0.368 (0.300) & 0.551 (0.336) \\
0.1 & 1.0 & 1.0 & -0.360 (0.258) & 0.258 (0.209) & 0.504 (0.487) & 0.480 (0.365) & 0.438 (0.527) \\
0.5 & 0.0 & 0.1 & -0.195 (0.425) & -0.050 (0.226) & -0.051 (0.132) & 0.134 (0.319) & -0.033 (0.776) \\
0.5 & 0.0 & 0.5 & -0.117 (0.267) & 0.075 (0.289) & 0.055 (0.278) & 0.199 (0.370) & -0.025 (0.773) \\
0.5 & 0.0 & 1.0 & -0.171 (0.369) & 0.006 (0.062) & 0.102 (0.272) & 0.079 (0.390) & -0.140 (0.764) \\
0.5 & 0.5 & 0.1 & -0.167 (0.426) & 0.070 (0.101) & 0.059 (0.131) & 0.151 (0.190) & 0.295 (0.492) \\
0.5 & 0.5 & 0.5 & -0.187 (0.400) & 0.060 (0.139) & 0.032 (0.146) & 0.091 (0.220) & 0.313 (0.458) \\
0.5 & 0.5 & 1.0 & -0.222 (0.457) & 0.023 (0.127) & 0.013 (0.266) & 0.004 (0.305) & -0.057 (0.534) \\
0.5 & 1.0 & 0.1 & -0.383 (0.313) & 0.157 (0.221) & 0.576 (0.875) & 0.526 (0.491) & 0.538 (0.188) \\
0.5 & 1.0 & 0.5 & -0.306 (0.299) & 0.229 (0.306) & 0.602 (0.702) & 0.512 (0.398) & 0.559 (0.327) \\
0.5 & 1.0 & 1.0 & -0.473 (0.451) & 0.092 (0.359) & 0.588 (0.510) & 0.198 (0.310) & 0.271 (0.458) \\
\bottomrule
\end{tabular}}
  \label{tab:app:all_semi_synthetic_results}
\end{table} 

\end{document}